%% file: main.tex
\documentclass[]{style/iromlab}

\input{preamble/notation}
\input{preamble/env}

\usepackage{enumerate}
\usepackage{cleveref}

\author[1*]{Zhiting Mei}
\author[1]{Tenny Yin}
\author[1]{Micah Baker}
\author[1*]{Ola Shorinwa}
\author[1]{Anirudha Majumdar}

\affiliation[1]{Princeton University}
\contribution[*]{Equal contribution.}

\begin{document}

\title{
{\LARGE World Models That Know When They Don't Know: \\}
{\Large Controllable Video Generation with Calibrated Uncertainty
}}

\abstract{
    \input{sections/abstract}
}


\website{
https://c-cubed-uq.github.io  %
}
{
c-cubed-uq.github.io   %
}

\code{
https://github.com/irom-princeton/c-cubed %
}
{
github.com/irom-princeton/c-cubed  %
}

\maketitle

\begin{figure}[th]
    \includegraphics[width=\linewidth]{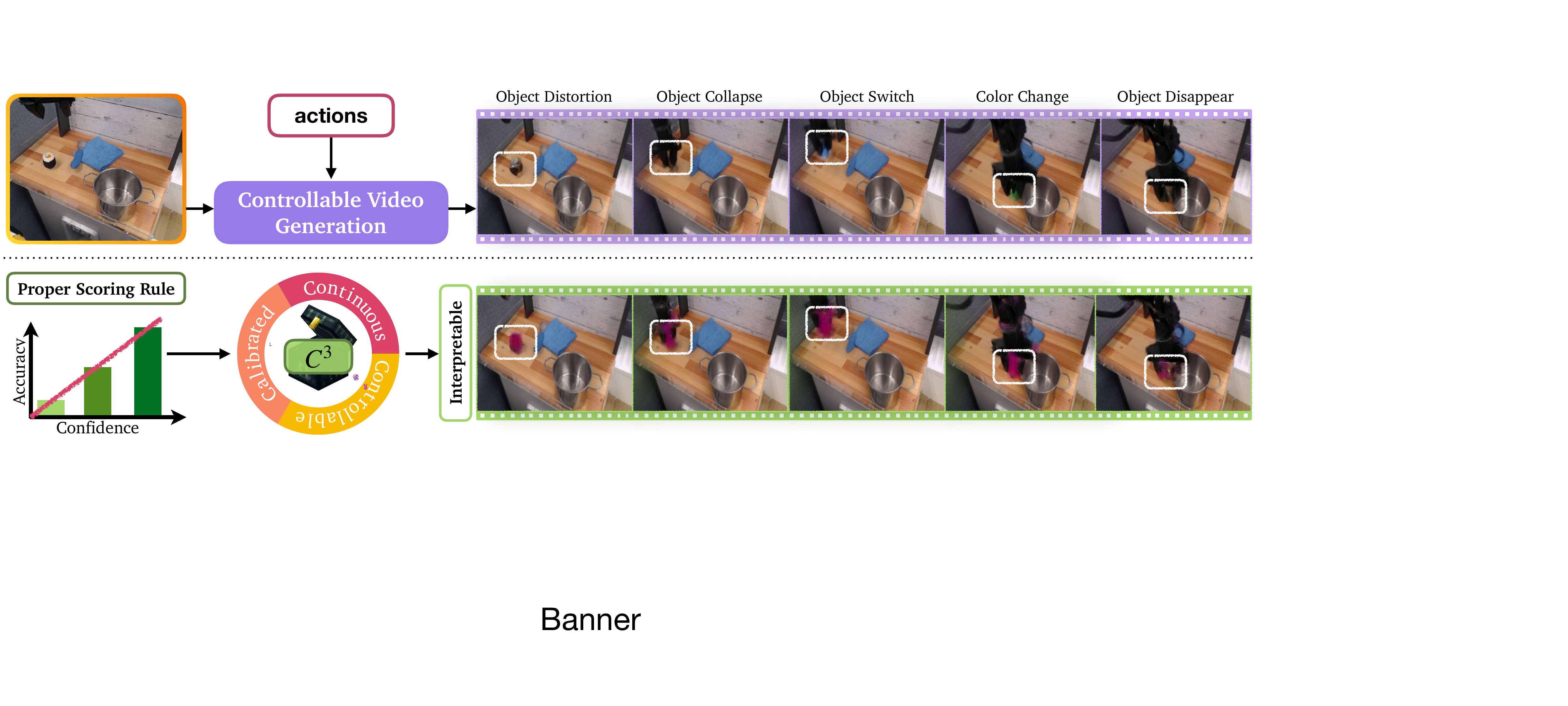}
    \caption{We present \algname, the first method for training video models that know when they don't know. Using proper scoring rules, \algname generates dense confidence predictions at the subpatch (channel) level that are physically interpretable and aligned with observations.
    }%
    \label{fig:banner}
\end{figure}

\input{sections/introduction}

\input{sections/related_work}

\input{sections/method}

\input{sections/evaluation}

\input{sections/evaluation_additional}

\input{sections/conclusion}

\input{sections/limitations_future_work}

\input{sections/acknowledgments}

\clearpage

\bibliographystyle{unsrtnat}
\bibliography{references.bib}

\clearpage

\beginappendix{
    \input{sections/appendix/appendix}

}

\end{document}

%% file: preamble/notation.tex
\DeclareDocumentEnvironment{example}{}{\noindent\textbf{Running example:}\itshape}{}

\usepackage{ifthen}
\newboolean{include-notes}
\newboolean{include-new}
\newboolean{include-remove}
\setboolean{include-notes}{true}
\setboolean{include-new}{false}
\setboolean{include-remove}{false}

\usepackage[dvipsnames]{xcolor}
\usepackage[normalem]{ulem}
\providecommand{\justin}[1]{\ifthenelse{\boolean{include-notes}}{\textcolor{orange}{\textbf{Jaime:} #1}}{}}

\providecommand{\princeton}[1]{\ifthenelse{\boolean{include-notes}}{\textcolor{orange}{#1}}{}}

\providecommand{\p}[1]{\smallskip \noindent \textbf{{#1}.}}

\providecommand{\mcal}[1]{\mathcal{#1}}
\providecommand{\mbb}[1]{\mathbb{#1}}

\providecommand{\mbf}[1]{\ensuremath{\mathbf{#1}}}
\providecommand{\Expect}{\ensuremath{\mathbb{E}}}

\providecommand{\norm}[1]{\ensuremath{\left\Vert #1 \right\Vert}}

\providecommand*{\acc}{\ensuremath{\mathrm{acc}}}
\providecommand*{\ex}{\ensuremath{\mathrm{e}}}

%% file: preamble/env.tex
\usepackage{multicol}
\usepackage{amsmath, amsfonts, amssymb, amsthm}
\usepackage{bm}
\usepackage[inline]{enumitem}
\usepackage{mathtools}
\usepackage{graphicx}
\usepackage{longtable,tabularx}
\usepackage{placeins} 
\usepackage{float}
\usepackage{multirow}
\usepackage{bbm}
\usepackage{threeparttable}
\usepackage{balance}
\usepackage[ruled,algo2e]{algorithm2e}
\usepackage{algpseudocode}
\usepackage{adjustbox}
\usepackage{booktabs}
\usepackage{bm}
\usepackage{etoolbox}
\usepackage{microtype}
\usepackage{units}
\usepackage{xspace}
\usepackage{tcolorbox}
\usepackage{wrapfig}
\usepackage{caption}
\usepackage{thm-restate}

\usepackage{soul}
\setuldepth{foobar}

\newtheorem{remark}{Remark}

\newbool{extended}
\setbool{extended}{false}

\makeatletter
\providecommand{\longdash}[1][2em]{%
  \makebox[#1]{$\m@th\smash-\mkern-7mu\cleaders\hbox{$\mkern-2mu\smash-\mkern-2mu$}\hfill\mkern-7mu\smash-$}}
\makeatother
\providecommand{\omitskip}{\kern-\arraycolsep}

\providecommand*{\algname}{\ensuremath{\mbox{$C^3$}}\xspace}
\providecommand*{\bfalgname}{\ensuremath{\mbox{$\mbf{C}^3$}}\xspace}
\providecommand*{\tcf}{\mbox{Continuous-scale binary classification model}\xspace}
\providecommand*{\mcf}{\mbox{Multi-class classification model}\xspace}
\providecommand*{\ncf}{\mbox{Fixed-scale classification model}\xspace}
\providecommand*{\tc}{\mbox{CS-BC}\xspace}
\providecommand*{\mc}{\mbox{MCC}\xspace}
\providecommand*{\nc}{\mbox{FSC}\xspace}
\usepackage{lineno}

%% file: sections/abstract.tex
Recent advances in generative video models have led to significant breakthroughs in high-fidelity controllable video synthesis, conditioned on text, robot actions, etc.
Despite their impressive capabilities, video models often \emph{hallucinate} --- generating future video frames that are misaligned with physical reality --- which raises serious concerns in many downstream applications. Exacerbating this issue, video models also lack the ability to assess and express their confidence, impeding hallucination mitigation.
To address this challenge, 
we propose 
\bfalgname,
an uncertainty quantification (UQ) method for training \emph{continuous-scale} \emph{calibrated} \emph{controllable} video models for dense confidence estimation at the subpatch level, precisely localizing the uncertainty in each generated video frame. Our UQ method introduces three core innovations to empower video models to estimate their uncertainty. First, our method develops a novel framework that trains video models for \emph{correctness} and \emph{calibration} via strictly proper scoring rules.
Second, we estimate the video model's uncertainty in latent space, avoiding training instability and prohibitive training costs associated with pixel-space approaches.
Third, we map the dense latent-space uncertainty  to \emph{interpretable} pixel-level uncertainty in the RGB space for intuitive visualization, providing high-resolution uncertainty heatmaps that identify untrustworthy regions.
Through extensive experiments on large-scale robot datasets (Bridge and DROID) and real-world evaluations, we demonstrate that our method not only provides calibrated uncertainty estimates within the training distribution, but also enables effective out-of-distribution detection.

%% file: sections/introduction.tex
\section{Introduction}
\label{sec:intro}
State-of-the-art (SOTA) controllable generative video models~\cite{agarwal2025cosmos, wan2025wan, peng2025open, deepmind_veo3_techreport} are capable of synthesizing high-fidelity videos with rich visual content across diverse task settings.
Video models offer significant promise as high-fidelity embodied world models that address some of the most important challenges in robotics, although their applications to robotics are still in their infancy.
Concretely, video models enable photorealistic simulation of complex dynamical interactions (e.g., simulation of deformable bodies) with the potential for continuous learning by scaling the training data~(See~\cite{mei2026video} for a recent survey on applications of video models in robotics).
However, video models have a high propensity to \emph{hallucinate}, i.e., to generate future video frames that are physically inconsistent,
posing a significant hurdle in applications that demand trustworthy video generation.
Despite their tendency to hallucinate, video models lack the fundamental capacity to express their uncertainty, which hinders their trustworthiness. To the best of our knowledge, only one existing work attempts to quantify the uncertainty of video models~\cite{mei2025confident}. However, the resulting estimates only capture task-level uncertainty, failing to resolve the model's uncertainty spatially and temporally at the frame-level.
Given that robotics applications broadly require fine-grained frame-level decision-making, we believe that \emph{dense spatio-temporal} uncertainty quantification is critical for practical adoption of video models in robotics.

To address this critical challenge, we present \bfalgname, an uncertainty quantification (UQ) method for \emph{calibrated controllable} video synthesis, enabling subpatch-level confidence prediction  at any resolution in video generation accuracy, i.e., at \textit{continuous} scales (see~\Cref{fig:banner}).
Our work is centered on three core contributions. First, we introduce a novel framework for training %
video models for both \textit{accuracy} and \textit{calibration}, founded on 
proper scoring rules as loss functions, effectively teaching video models to quantify their uncertainty during the video generation process. We demonstrate that the resulting uncertainty estimates are \textit{well-calibrated} (i.e., neither underconfident nor overconfident) %
using benchmark robot learning datasets, e.g., the Bridge~\cite{walke2023bridgedata} and DROID~\cite{khazatsky2024droid} datasets. 

Second, we derive our UQ method directly in the \textit{latent space} of the video model. This key design choice circumvents the high computation %
costs associated with video generation in the (higher-dimensional) pixel space. Further, operating in the latent space streamlines applicability of our proposed method to a wide range of SOTA latent-space video model architectures~\cite{agarwal2025cosmos, wan2025wan, peng2025open}, without requiring specialized knowledge or adaptation for implementation.
Moreover, we compute \textit{dense} uncertainty estimates at the subpatch 
level for high-resolution UQ, with more fine-grained detail compared to patch-level UQ representations.

Third, we decode latent-space uncertainty into \emph{interpretable} pixel-space confidence estimates via temporal RGB heatmaps for intuitive visualization. We show that the uncertainty heatmaps are well-aligned with physical intuition, with areas of greater uncertainty associated with hallucinations --- highlighting untrustworthy areas of the video. Further, we show that the model's confidence estimates are negatively correlated with the error between the generated video and the ground-truth video, which is also consistent with intuition.

Finally, we demonstrate the effectiveness of \algname in detecting \textit{out-of-distribution} inputs (i.e., environment conditions and actions) in video generation through real-world experiments on a WidowX~250 robot. 
Particularly, we show that \algname provides calibrated uncertainty estimates, even when the quality of the generated video is significantly compromised given the distribution shift at test time.

%% file: sections/related_work.tex
\section{Related Work}
\label{sec:related_work}
\p {Video Generation Models}
Research breakthroughs in video generation have led to significant advances in the capabilities of generative video models in recent years. While early video generation models were limited to generating short-duration videos (only a few frames) with small temporal changes, SOTA models can generate seconds-long videos (hundreds of frames) with impressive photorealistic detail. Early methods \cite{jia2016dynamic, finn2016unsupervised, liu2017video} synthesize novel videos by applying local (pixel-level) transformations to input images, composing these transformations to capture more complex temporal changes. However, these methods are limited to localized scene updates generally centered around a target object in the scene video. Moreover, these methods lack sufficient expressivity to generate photorealistic videos. Subsequent work \cite{clark2019adversarial, vondrick2016generating, lee2018stochastic} employs generative adversarial networks  \cite{goodfellow2014generative}, 
while others~\cite{babaeizadeh2017stochastic, babaeizadeh2021fitvid, wu2021greedy} utilize variational inference with variational autoencoders \cite{kingma2013auto} for video generation but fail to sufficiently address limited expressiveness and mode collapse.
Addressing these limitations, SOTA methods \cite{wan2025wan, agarwal2025cosmos, peng2025open, kong2024hunyuanvideo} leverage diffusion-based or flow-based modeling \cite{ho2022video, ho2020denoising, lipman2022flow} for high-fidelity video generation.
More recent work~\cite{agarwal2025cosmos, quevedo2025worldgymworldmodelenvironment, guo2025ctrlworldcontrollablegenerativeworld} builds upon these methods to enable video generation conditioned on robot actions.

\p{Uncertainty Quantification of Video Models}
Uncertainty quantification of large language models (LLMs) has been been extensively studied (see \cite{shorinwa2025survey} for a review of UQ methods for LLMs); however, only a few papers have explored uncertainty quantification of image or video generation models. Like traditional methods for UQ of deep neural networks \cite{abdar2021review}, prior work on UQ of generative image models applies Bayesian methods to image diffusion models to estimate epistemic and aleatoric uncertainty using a variance-decomposition-based approach \cite{chan2024estimating}. Another approach \cite{berry2024shedding} takes an ensemble-based UQ perspective, estimating the uncertainty of image diffusion models using the mutual information over the distribution of the weights of an ensemble of the diffusion models. Some other methods \cite{franchi2025towards} extract language descriptions from the generated images, facilitating uncertainty quantification of image generation models using established UQ methods for language models. Extending these methods to uncertainty quantification of video generation models is not trivial, given the spatio-temporality of videos and the significant computation costs of classical UQ methods. One existing approach \cite{mei2025confident} directly considers uncertainty quantification of video models. However, this method only provides a composite confidence estimate for each generated video and thus fails to provide more informative \textit{dense} confidence estimates at the frame-level or pixel-level. Robotics applications generally require fine-grained frame-level decision-making, which benefits from dense spatio-temporal uncertainty quantification. Our work explores this research direction specifically.

%% file: sections/method.tex
\section{Uncertainty Quantification of Video Models}
\label{sec:method}

\begin{figure*}[t]
    \centering
    \includegraphics[width=\linewidth]{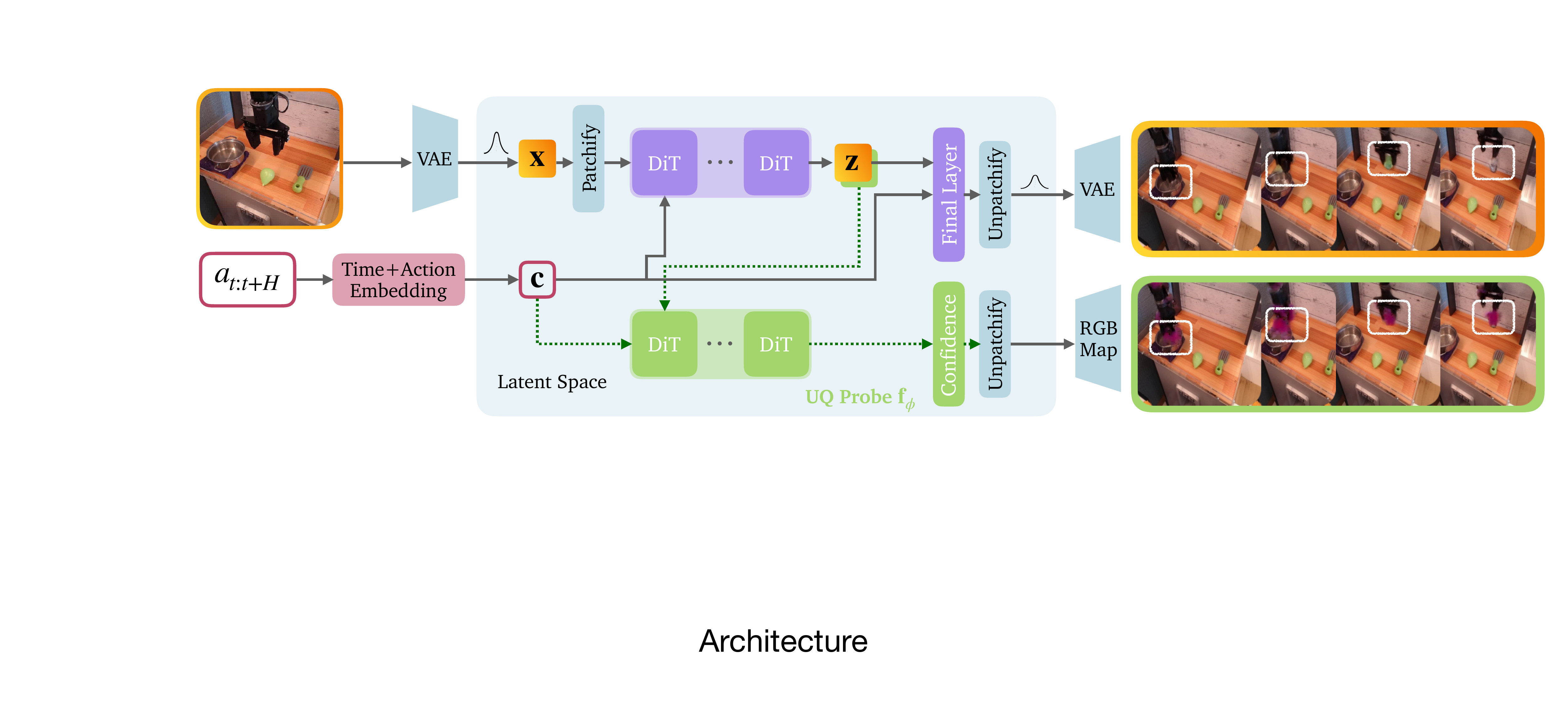}
    \caption{\textbf{Model Architecture.} \algname enables simultaneous video generation and uncertainty quantification, quantifying the model's confidence in its accuracy using a UQ probe acting on the video latents. High uncertainty regions (red) show hallucinations.}
    \label{fig:architecture}
\end{figure*}

For simplicity, we limit the discussion of our UQ method to video diffusion or flow-based models, given their SOTA performance.
We provide a brief review of these models in Appendix~\ref{app:diffusion}.
However, we note that our proposed method readily applies to other video model architectures, such as GAN-based/RNN-based video models, with relatively straightforward adaptations. 
We adopt the latent diffusion transformer (DiT) architecture, described by:
\begin{align}
   \mathbf{x} = \text{Encode (\mbf{v}, \mbf{g})}, \quad
   \mathbf{\hat x} \sim \text{DiT}(\mathbf{x}, \mathbf {a}), \quad
   \hat{\mbf{o}} = \text{Decode} (\mathbf{\hat x}), \label{eq:dit}
\end{align}
where $\mbf{v}$ denotes the input video frames, $\mbf{g}$ denotes other conditioning inputs (e.g., text or action),  $\mathbf{x}\in \mathcal{U}$ denotes the encoded conditioning inputs in the latent video space $\mathcal{U}$, and $\mathbf a\in\mathcal A$ represents the action sequence carried out by the agent.

\subsection{Confidence Prediction}
\label{sec:confidence_prediction}

We introduce \algname, a method for uncertainty quantification of controllable (action-conditioned) video generation models that provides dense estimates of the model's confidence in the accuracy of each video frame at the subpatch level, conditioned on input actions.
Concretely, we train the video model 
$\mathcal V_\theta: \mathcal U\times \mathcal{A} \rightarrow \mathcal U\times \mathcal U$
to generate accurate video frames with corresponding dense confidence estimates: %
\begin{equation}
    \mathbf{\hat x}, \mbf{\hat q} \sim \mathcal V_{\theta}(\mathbf{v}, \mathbf{g}, \mathbf {a}),
\end{equation}
where $\mathbf{\hat q}\in \mathcal{U}$ is the confidence prediction. Each element in $\mathbf{\hat q}$ corresponds to the model's confidence in the accuracy of the associated subpatch. %

Traditional UQ methods, such as Monte Carlo-based methods or ensemble-based techniques, generally require multiple forward passes or multiple instances of the model to estimate uncertainty. However, video models typically have billions of parameters, making these methods too computationally expensive and generally intractable. 
To overcome these challenges, we take a novel approach to UQ of video models. First, we pose UQ as a classification problem over the accuracy of the generated video, seeking to assess the model's confidence in the video accuracy.
This key choice eliminates inductive biases associated with restricting the predicted accuracy to a specific class of probability distributions, which could hinder calibration. %

Given the high computational cost of video generation, we design a transformer-based UQ probe ${\mbf{f}_{\phi}: \mcal{U} \rightarrow \mcal{U}}$ to estimate the video model's confidence directly in \emph{latent space}.
We integrate $\mbf{f}_{\phi}$ within the video generation framework for simultaneous video generation and uncertainty quantification during training and inference, as summarized in~\Cref{fig:architecture}. However, we note that both components can also be trained independently.
For more efficient training, we generate the videos in latent space using a vector-quantized variational autoencoder (VQ-VAE) with spatio-temporal convolution and attention layers to map input video frames to a lower-dimensional latent space. Specifically, we roll-out the forward and reverse diffusion processes in latent space, before mapping the generated latent video to the pixel space using a decoder. In this work, we utilize pre-trained VQ-VAEs \cite{blattmann2023align, kondratyuk2023videopoet, agarwal2025cosmos}, self-supervised with a reconstruction objective to compress RGB videos into a compact latent space.
Further, we leverage diffusion forcing \cite{chen2024diffusion} for independent per-sample noise schedules, in line with prior work. 

For action-conditioned video generation, we compute action embeddings from input actions using a multi-layer perceptron and sum the resulting action embeddings with the timestep embedding computed using frequency-space encodings. We feed the resulting embeddings $\mbf{c}$ to the DiT. 
Given the input video frame and action, we extract the internal features $\mbf{z}$ of the DiT from the penultimate layer, which is passed into $\mbf{f}_{\phi}$, alongside the action and timestep embedding, to predict the subpatch (channel-wise) confidence ${\hat{\mbf{q}}}$. This confidence represents the probability that each subpatch of the generated latent video is accurate with respect to an element-wise boolean function $\bm{\acc}$, which we elaborate in~\Cref{sec:architecture}.

\subsection{Model Architectures}
\label{sec:architecture}
The definition of the accuracy function $\bm{\acc}$ induces different model architectures.
To demonstrate our method's amenability to different realizations, we consider three possible architectural instantiations of \algname. %
We define $\bm{\acc}$ in terms of a distance function $\mbf{d}$. In this work, we use the $L_1$ loss: 
\begin{equation}
    \label{eq:latent_video_distance}
    \mbf{d}(\mbf{\hat x}, \mbf{x}^{\star}) \coloneqq \lvert \mbf{\hat x} - \mbf{x}^{\star} \rvert,
\end{equation}
although other distance metrics can also be used, e.g., the squared deviation.
However, we emphasize that in our setting, all $p$-norms simplify to the $L_1$ loss in~\Cref{eq:latent_video_distance} since $\mbf{d}$ is applied element-wise, making them equivalent. We use the $L_1$ loss for simplicity.
Given $\mbf{d}$, the accuracy function maps the generated videos to a binary-valued output space of the same dimensions as $\mcal{U}$, where each element is in $\{0, 1\}$, given by the boolean operator:
\begin{equation}
    \label{eq:latent_video_accuracy}
    \bm{\acc}(\mbf{\hat x}, \mbf{x}^{\star}, \varepsilon) \coloneqq \mbf{d}(\mbf{\hat x}, \mbf{x}^{\star}) \leq \varepsilon,
\end{equation}
based on the errors between the ground-truth and generated videos, given a threshold $\varepsilon$.
The technique used in specifying $\varepsilon$ induces a range of model architectures, namely: (i) fixed-scale classification models, (ii) multi-class classification models, and (iii) continuous-scale classification models, which we describe in the following subsections. We train all variants of our model with proper scoring rules for calibration. Appendix~\ref{app:proper_scoring_rule} provides a brief overview of proper scoring rules.

\p{\ncf  (\nc)}
The \nc model predicts the accuracy of generated videos at a fixed accuracy resolution during training and inference, and thus requires the specification of a single error threshold $\varepsilon$.
By requiring only a single value of $\varepsilon$, \nc models are typically faster to train than other models at the cost of generality to a range of resolutions.
In practice, we select an appropriate value of $\varepsilon$ for the task domain.
As is standard in classification problems, we train $\mbf{f}_{\phi}$ to predict the logits and use the sigmoid function $\sigma$ to map these values to valid confidence estimates $\mathbf{\hat q}$ that lie within $[0, 1]$: %
\begin{equation}
    \mathbf{\hat q} = \sigma(\mathbf f_\phi(\mathbf z, \mathbf c)),
    \label{eq:sigmoid}
\end{equation}
where $\mathbf f_\phi$ is the confidence probe, $\mathbf z$ is the latent internal feature, and $\mathbf c$ is the latent action/time embedding. 
We optimize the parameters of $\mathbf f_\phi$ with the Brier score loss function, given by: ${\text{BS} = \Expect_{y}(\hat{q} - y)^2}$ with ground-truth accuracy $y$ and predicted confidence $\hat{q}$ for the prediction over a single subpatch. We sum over all subpatches in computing the loss for each video. With a slight overload in notation, ${y,\hat{q}}$ denote each component in ${\mbf{y}, \mbf{\hat{q}}}$, respectively.

\p{\mcf (\mc)} 
Inspired by the effectiveness of classical UQ methods for LLMs~\cite{shorinwa2025survey}, we pose video model UQ as a multi-class classification problem by discretizing the output space of predictions into confidence bins, with the corresponding $\bm{\acc}$ defined by:
\begin{equation}
    \label{eq:latent_video_accuracy_bin}
    \bm{\acc}(\mbf{\hat x}, \mbf{x}^{\star}, O_i) \coloneqq \varepsilon_i \leq\mbf{d}(\mbf{\hat x}, \mbf{x}^{\star}) < \varepsilon^i,
\end{equation}
where $O_i$ represents the $i$-th bin with lower-bound  $\varepsilon_i$ and upper-bound $\varepsilon^i$.
For each subpatch of the generated video, the \mc model predicts its confidence that the corresponding subpatch is accurate with respect to the accuracy thresholds associated with the bin.
Like the \nc model, the \mc model predicts the logits for each bin, which is subsequently mapped to valid confidence (probability) values $\mathbf{\hat q}$ using the softmax.
We optimize $\mathbf f_\phi$ with the cross-entropy loss function, which is a strictly proper scoring rule given by: ${\text{CE} = \Expect_{y}\big[-\sum_{k} y_k \log q_k\big]}$, with ground-truth value $y$ and predicted confidence $q_k$ for the $k$-th bin.

\p{\tcf  (\tc)}
To demonstrate the expressiveness of our approach, we train a continuous-scale model for any-resolution confidence prediction, conditioning $\mathbf f_\phi$ on an accuracy threshold $\varepsilon$ specified at inference.
During training, we uniformly sample a set of $\varepsilon_v$ at each iteration to ensure sufficient coverage of the thresholds. ($\varepsilon_v$ is the linearly scaled version of $\varepsilon$ specifying the deviation between predicted and ground-truth deviations. See Appendix~\ref{app:diffusion} for the full derivation.)
In practice, for faster training, we discretize the space of $\varepsilon$ using an adaptive hierarchical technique, by first dividing $\varepsilon$ into uniform bins and further adaptively subdividing bins for higher resolution. When domain knowledge is available, more informed discretization or sampling schemes can be used for more efficient training. 
Like the preceding models, the \tc model predicts logits for each subpatch, which is mapped to confidence estimates using the sigmoid function:
\begin{equation}
    \mathbf{\hat q} = \sigma(\mathbf f_\phi(\mathbf z, \mathbf c, \varepsilon)),
    \label{eq:conf_tc}
\end{equation}
taking the conditioning input threshold $\varepsilon$.
In our experiments, we explore training the models with both Brier score and binary cross entropy, given by: ${\text{BCE} = \Expect_{y}\big[y\log q + (1 - y) \log(1 - q)\big]}$,
which are proper scoring rules. %

\p{End-to-end training}
We train the video generation and uncertainty quantification modules independently end-to-end using the loss function: 
\begin{equation}
    \label{eq:model_loss}
    \mcal{L}_{\theta, \phi} = \mcal{L}_{\theta} + \mcal{L}_{\phi},
\end{equation}
where $\theta$ represents the parameters of the DiTs for video generation and $\phi$ represents the parameters of the UQ probe. 
We apply a stop-gradient operator between the video generation DiTs and the UQ probe. 
In our ablations in Appendix~\ref{app:ablations}, we explore the effects of backpropagating gradients from the UQ probe $\mathbf f_\phi$ to the video generation DiTs.
To optimize the loss function, we use stochastic gradient descent with a cosine-annealing decay schedule applied to the learning rate. 

\begin{restatable}[Uncertainty Decomposition]{proposition}{propdecomp}
\label{prop:calibration}
    Given the input actions and video frames, the predicted confidence $\hat{\mbf{q}}$ provides a calibrated measure of uncertainty of the video diffusion model in the generated video, provided that $\phi$ converges to an optimal solution.
\end{restatable}

We provide the proof in Appendix~\ref{app:proofs}.

\subsection{Decoding Latent Confidence Predictions}
Like the latent video $\mbf{x}$, the predicted confidence $\hat{\mbf{q}}$ is not immediately interpretable; hence, we decode the predicted confidence from the latent space to the pixel (RGB) space for better visualization. However, simply utilizing pre-trained video tokenizers as decoders would generally yield equally uninterpretable outputs, since these pre-trained decoders are trained to map RGB embeddings from the latent space to the pixel space. 
As a result, we define a color map in latent space by encoding monochromatic RGB video frames into the latent space. For simplicity, we construct a latent color map from red-only, green-only, and blue-only video frames; however, higher-resolution color maps can also be constructed. We map the confidence estimates to latent RGB video frames by interpolating between the video frames in the latent color map. 
Subsequently, we map the latent RGB video frames for the predicted confidence to pixel space using the same tokenizer used in decoding the latent video $\mbf{x}$.
In~\Cref{sec:evaluation}, we demonstrate that the resulting uncertainty heatmaps are well-aligned with intuition, identifying areas of the generated video that contain hallucinations.

%% file: sections/evaluation.tex
\section{Experiments}
\label{sec:evaluation}
We evaluate the performance of \algname in uncertainty quantification of action-conditioned video models, specifically examining its calibration, interpretability, and out-of-distribution detection capabilities via the following questions:
(i) Is \algname \textit{underconfident}, \textit{calibrated}, or \textit{overconfident}?
(ii) Are \algname's uncertainty estimates interpretable?
(iii) Can \algname detect \textit{OOD} inputs at inference?
We provide additional experiments and ablations in the Appendix.

\p{Datasets}
We conduct experiments on the Bridge dataset~\cite{walke2023bridgedata}, a standard benchmark dataset for robotics-oriented video models. The Bridge dataset consists of real-world robot trajectories collected in 24 environments on a WidowX 250 robot arm with a fixed RGB camera, capturing broad environment variations across different robot manipulation tasks. In addition, we present additional results using the DROID dataset~\cite{khazatsky2024droid} in Appendix~\ref{app:droid}. The DROID dataset consists of trajectories collected on a Panda robot arm with a Robotiq gripper, featuring greater coverage of tasks with multi-view camera observations collected using a wrist camera and two scene cameras. 

\p{Metrics}
In order to empirically evaluate the calibration of \algname, we use metrics that capture deviation from perfect calibration, including expected calibration error (ECE) and maximum calibration error (MCE), see \Cref{eq:proof_uncertainty_calibration}.

\subsection{Are \texorpdfstring{\algname}{C3}'s uncertainty estimates calibrated?}
\label{sec:evaluation_calibration}
We examine the calibration of the uncertainty estimates computed by \algname in
dynamics prediction with controllable video models, specifically in robot manipulation tasks which constitutes a major application domain for these models.
We train the three model architectures: \tc, \mc, and \nc (discussed in  \Cref{sec:architecture}) on the train dataset split of the Bridge dataset and evaluate the trained models on the test split. To assess calibration, first, we generate videos and their corresponding dense confidence predictions conditioned on the initial video frame and the entire action trajectory for each sample in the test dataset. Subsequently, we compute the ECE and the MCE for each model, measuring the deviations from perfect calibration. 
We provide additional details on the evaluation procedure in Appendix~\ref{app:eval_procedure}.

\begin{wrapfigure}[17]{r}{0.47\textwidth}
    \vspace{-6ex}
    \centering
    \includegraphics[width=\linewidth]{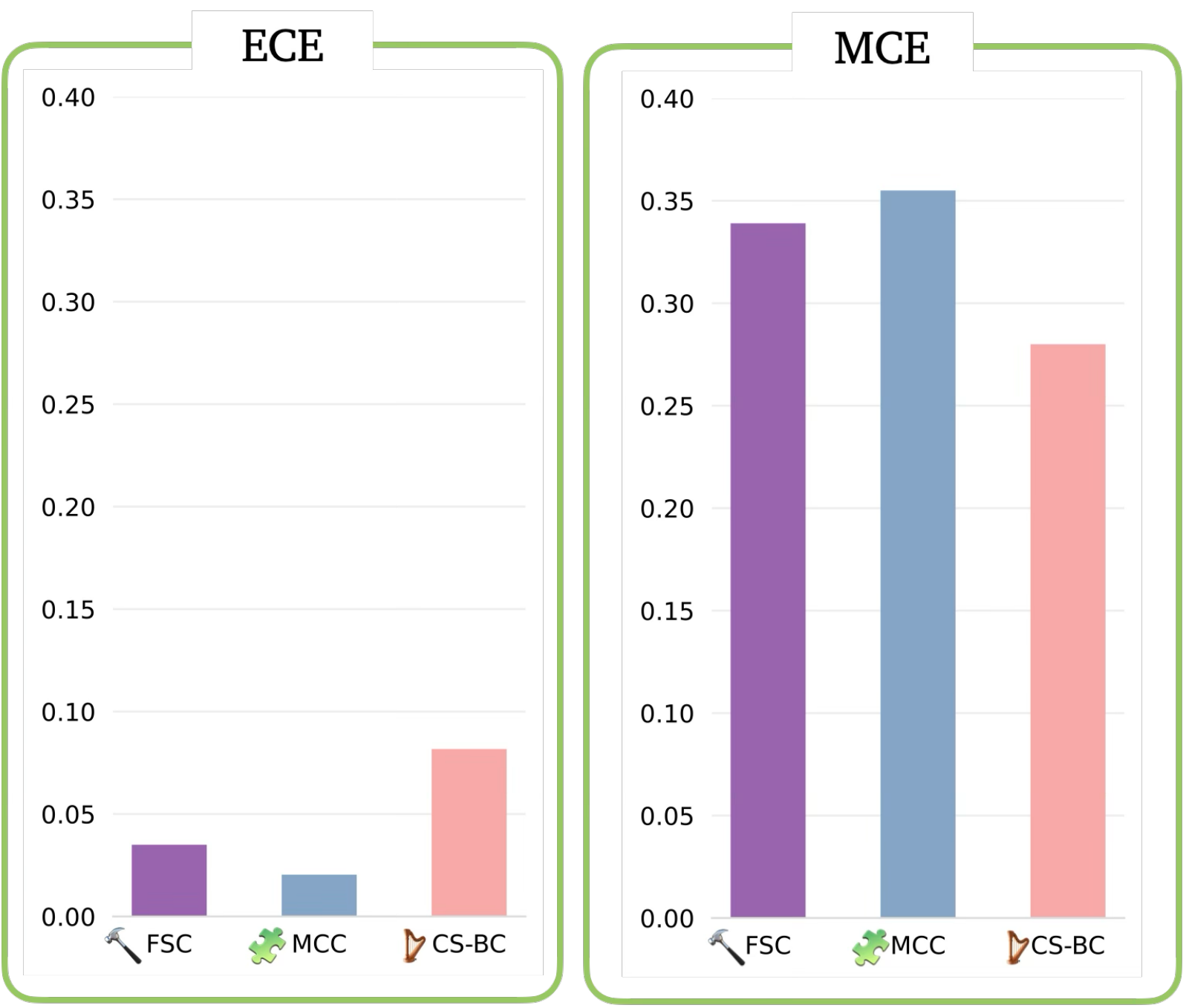}
    \caption{\textbf{Average calibration error.} All three architectures have low ECE and MCE.
    }
    \label{fig:average_calibration}
\end{wrapfigure}
\p{Calibration Errors}
In \Cref{fig:average_calibration}, we show the average ECE and MCE of each model across all the test videos. 
For the continuous-scale model \tc, we compute the average errors across ten equally-spaced error thresholds 
$\varepsilon_v$, spanning the observable latent-space prediction error domain as visualized in~\Cref{fig:latent_space_error}.
In contrast, the \nc model does not take in an accuracy scale for conditioning; as a result, we compute the ECE and MCE at the fixed-scale used in training the model. (For more informative evaluation and more comparative results, we select the fixed scale to lie within the range of $\varepsilon_v$ used by the \tc model.)  Similarly, we compute the ECE and MCE for all classes (accuracy scales) in the \mc model and report the average values in~\Cref{fig:average_calibration}. 
The results indicate that \algname \mbox{produces} \textit{well-calibrated} uncertainty estimates across all models. 
Although all models achieve relatively the same MCE, their performance on the ECE differs. This small difference can be explained by the tradeoff between continuous-scale calibration and fixed-scale calibration. The improved expressiveness and flexibility of continuous-scale calibration might come at the cost of a marginal reduction in calibration at a single (specific) scale. The converse holds for the fixed-scale model, which is less expressive and flexible. The multi-class classification variant lies between these extremes on the tradeoff curve. We emphasize that the superior calibration performance of \algname arises from the use of proper scoring rules.

\begin{figure}[t]
    \centering
    \includegraphics[width=0.85\linewidth]{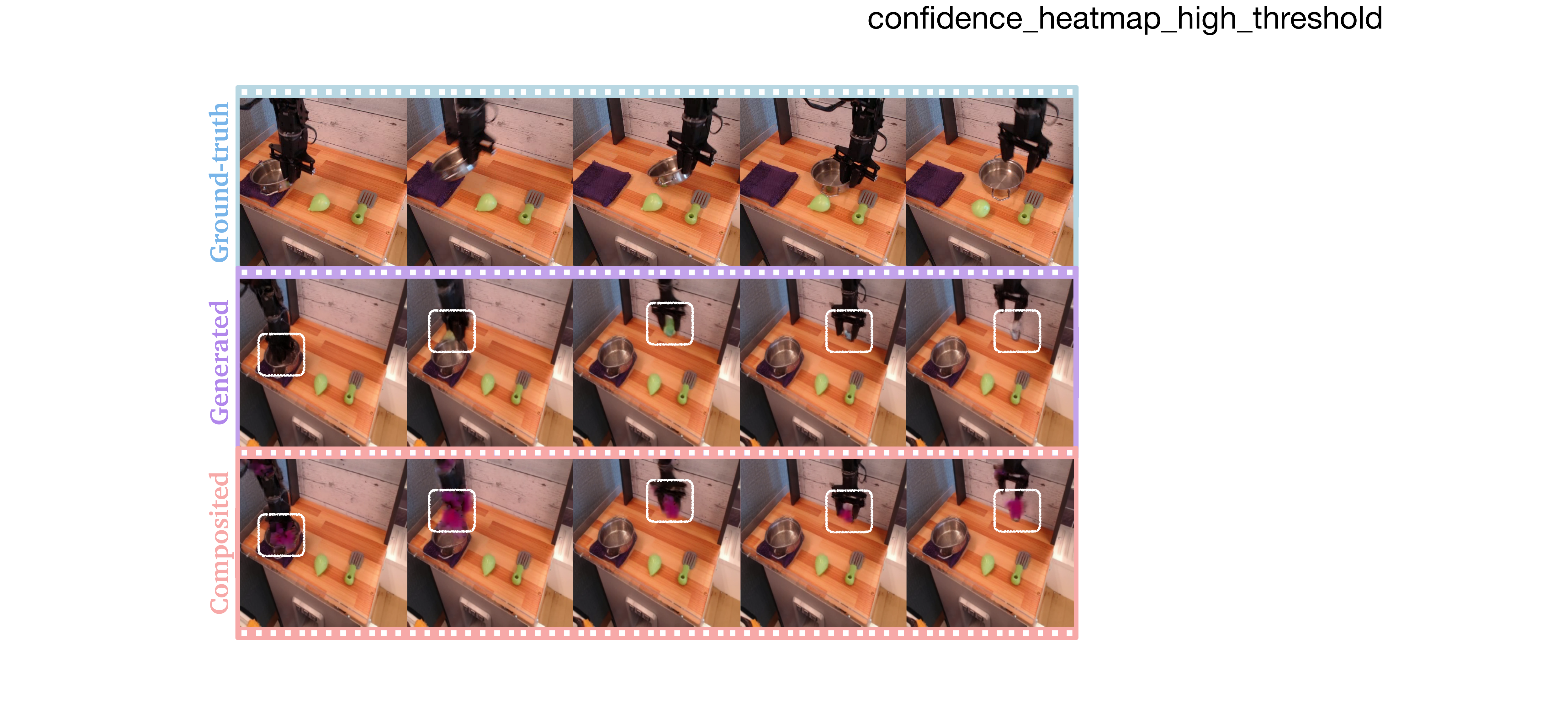}
    \caption{\textbf{Interpretability of \algname's confidence estimates.} 
    While attempting to pick up the pot, the video model hallucinates a green object appearing within the robot's gripper in a way that violates the law of physics. The object further undergoes non-casual deformation and changes to its color. \algname localizes these hallucinations, highlighting that the video model is highly uncertain about these hallucinations.
    }
    \label{fig:confidence_heatmap}
\end{figure}

\begin{figure}[th]
    \centering
    \includegraphics[width=0.9\linewidth]{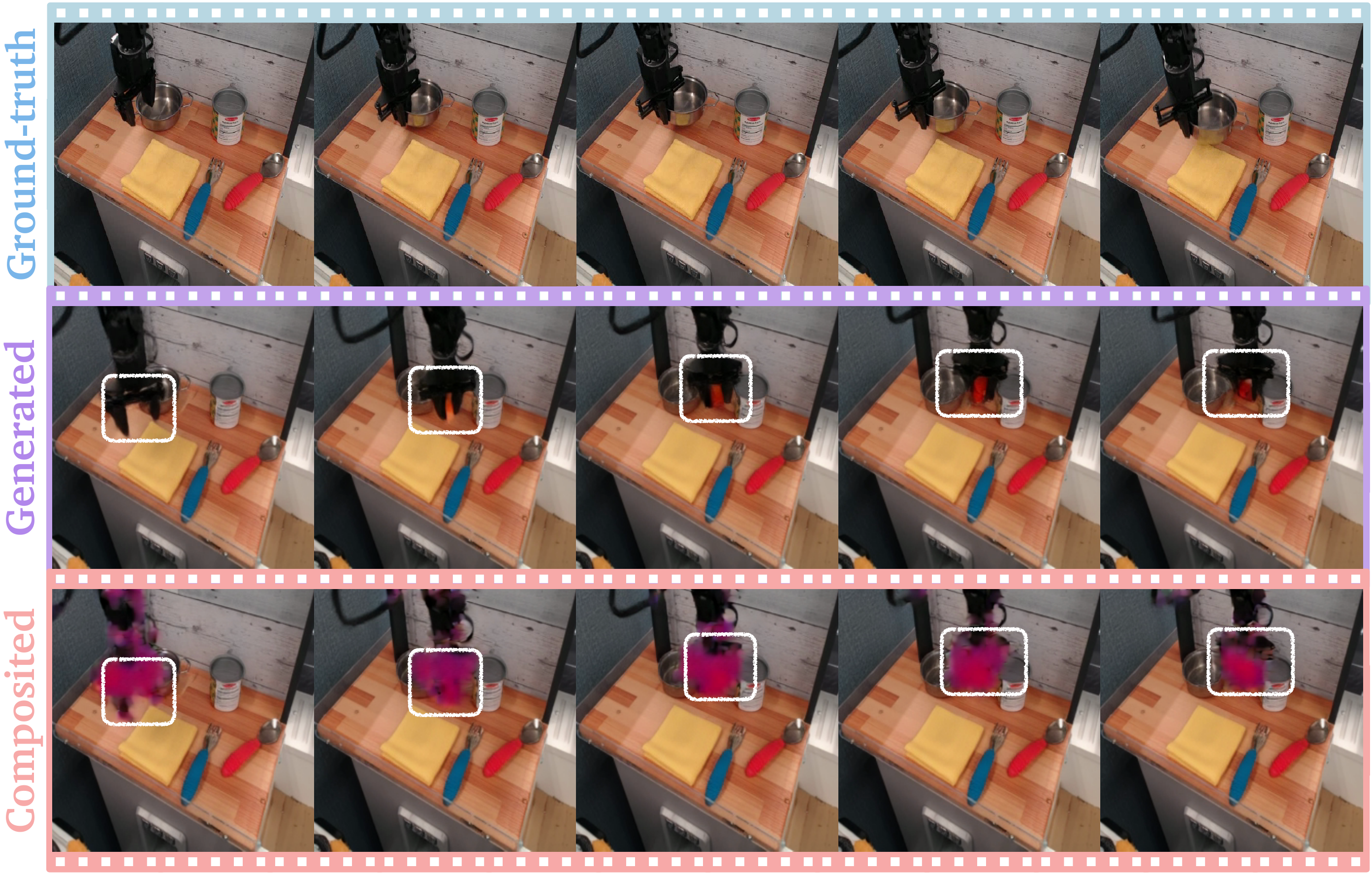}
    \caption{\textbf{Capturing Hallucinations.} The video model hallucinates the robot picking up a carrot, which appears out of nowhere. \algname is able to capture the high uncertainty corresponding to these hallucinations.}
    \label{fig:conf_heatmap_artifacts}
\end{figure}

\begin{figure}
\centering
\begin{minipage}[t]{0.49\linewidth}
    \centering
    \includegraphics[width=\linewidth]{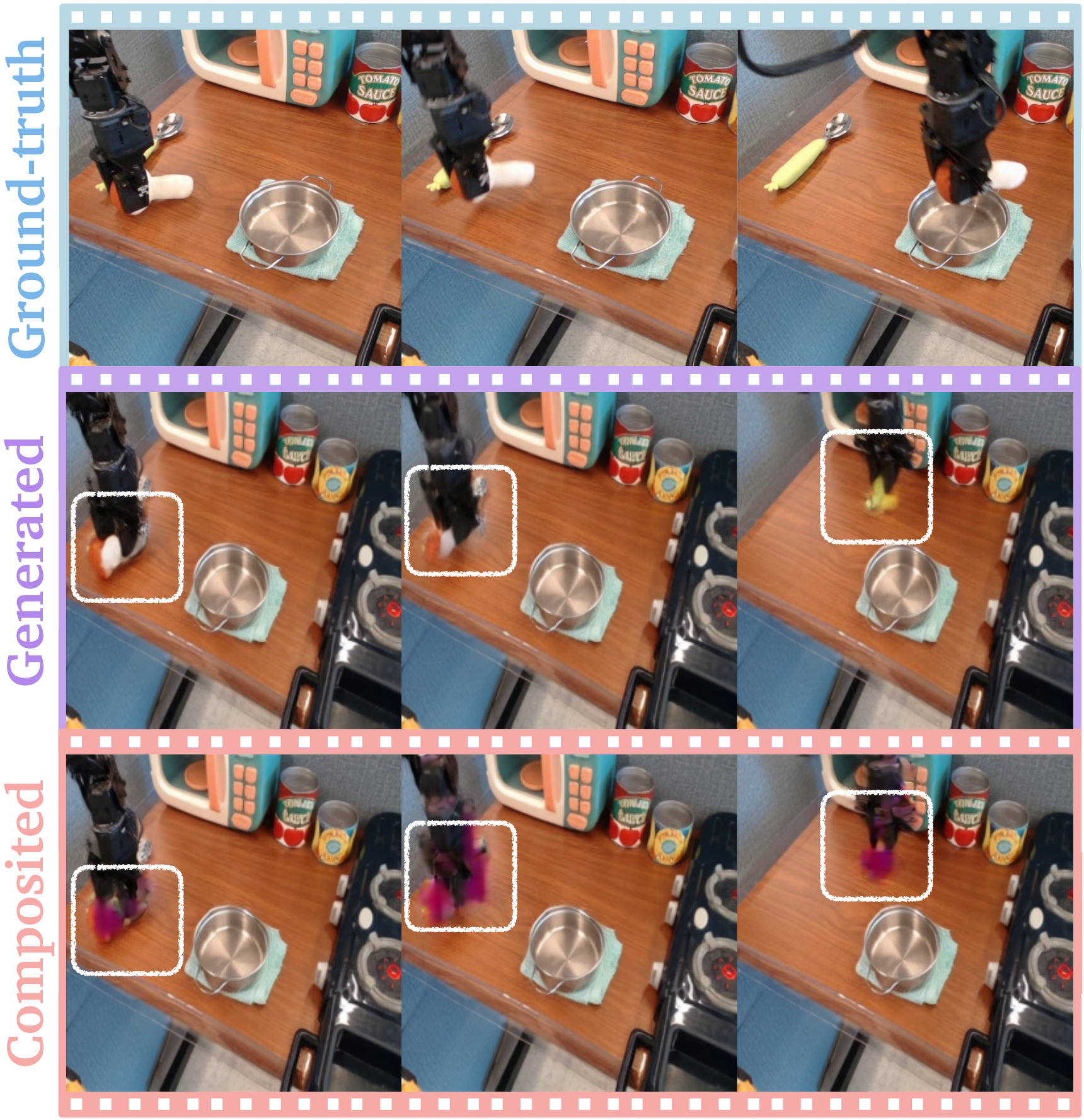}
    \caption{\textbf{Uncertainty in Object Interactions.} The plushy toy deforms in non-casual ways that violate physical laws. \algname shows that the video model is highly uncertain about these interaction dynamics.}
    \label{fig:red_around_grasped_object}
\end{minipage}\hfill
\begin{minipage}[t]{0.49\linewidth}
    \centering
    \includegraphics[width=\linewidth]{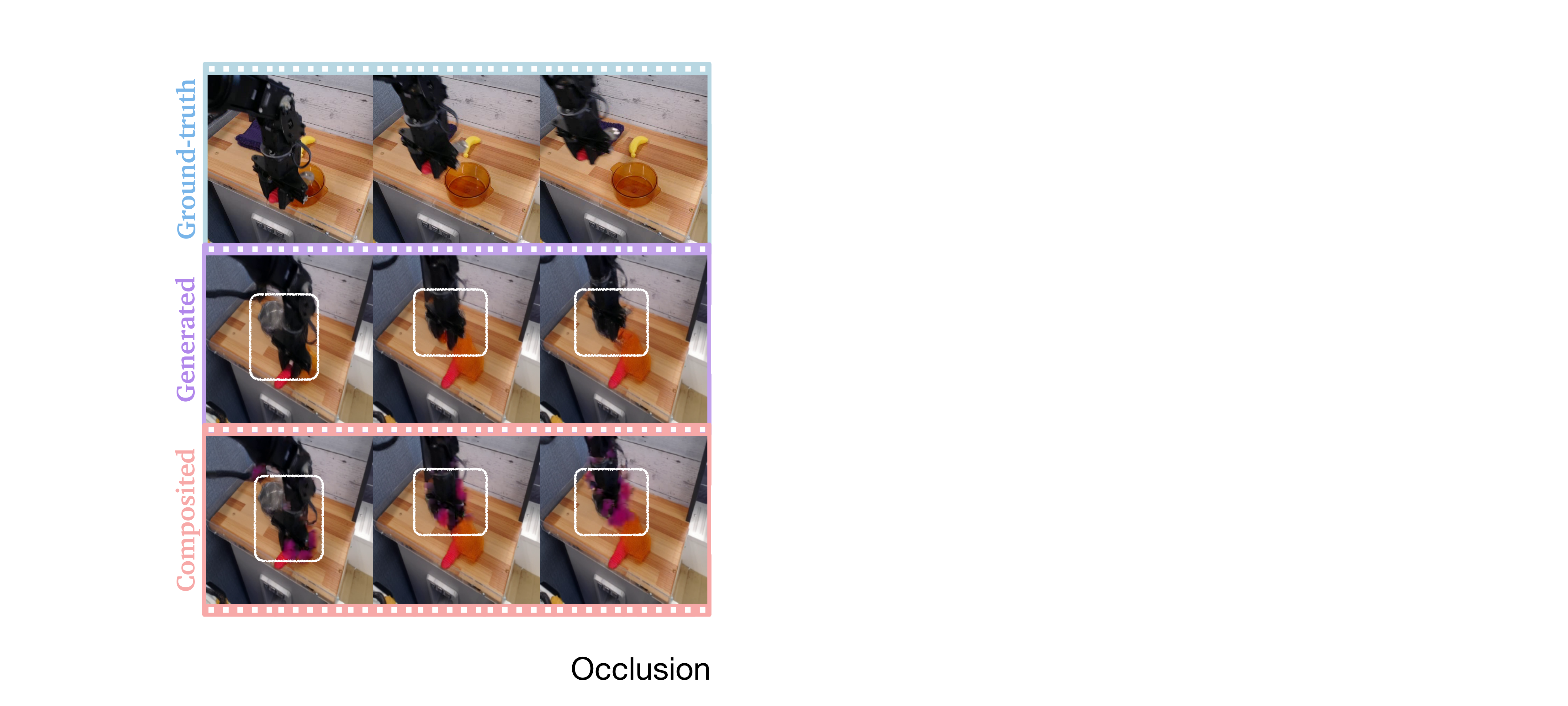}
    \caption{\textbf{Occlusions.} The robot arm occludes some areas of the scene while grasping the red spoon. \algname identifies the occluded areas as high-uncertainty regions, in line with intuition.}
    \label{fig:occlusion}
\end{minipage}
\end{figure}

\subsection{Are \texorpdfstring{\algname}{C3}'s uncertainty estimates interpretable?}
\label{sec:exp_interpretable}
Here, we examine the interpretability of the video model's uncertainty estimates in video trajectory prediction in robot manipulation.

\p{Qualitative Results}
In~\Cref{fig:confidence_heatmap}, we show the ground-truth and generated videos from the video model, along with a visualization of the video model's confidence using a confidence heatmap, which transforms the model's confidence predictions to the RGB color space using a color map. 
In the composited uncertainty map, the \emph{red} regions represent areas of \emph{high uncertainty}, corresponding to locations where the model is unsure if the generated video matches the ground-truth video.

In~\Cref{fig:confidence_heatmap},
we observe that the video model is uncertain about the robot's interaction with the pot in the video. After attempting to grasp the pot, the video model hallucinates a green object appearing within the robot's gripper. As the interaction proceeds, the object morphs in non-physical ways with unrealistic changes to its shape and color. \algname identifies these hallucinations, revealing that the video model is highly uncertain about its hallucinations, as indicated by the red regions in the composited uncertainty map. 
These results underscore the interpretability of our proposed UQ method.

\p{Quantitative Results}
We assess the correlation between the estimated confidence of the video model and the error between the ground-truth and generated latent videos using the \textit{Shepherd's Pi} correlation~\cite{rousselet2012improving}, which is a robust correlation method that uses bootstrapping to identify outliers that would otherwise skew the correlation coefficient. 
For calibrated models, one would expect a negative correlation between the estimated confidence and the error between the ground-truth and generated videos, generally indicating an increase in the uncertainty of the video model as the video error increases.
As expected, for the \nc and \tc models, we observe a statistically significant \textit{negative} correlation of $-0.373$ and $-0.172$ between the confidence estimates and the absolute errors in the generated video at a $99\%$
significance level, respectively. However, for the \mc model, we obtain a positive correlation coefficient due to the inadequate supervision of rightmost bins which correspond to greater latent error values, given that most of the generated video patches have a notably small latent error.
For a more informed analysis, we examine the correlation of the confidence estimates of the \mc model with the maximum bin edge set at $0.2$. We find that the confidence error is negatively correlated  with the video error at the $99\%$ significance level, with a coefficient of~${-0.130}$. %

We provide additional visualizations showing the remarkable ability of \algname in capturing hallucinations, i.e., in localizing regions of the generated video where the model inserts artifacts such as previously non-existent objects or morphed objects.
In~\Cref{fig:conf_heatmap_artifacts}, the video model hallucinates the robot picking up a carrot out of nowhere. \algname localizes the corresponding areas as regions of high uncertainty.
Similarly, \algname reveals that the video model is uncertain about the interaction dynamics of the plushy toy in~\Cref{fig:red_around_grasped_object}. After the initial grasp, the toy deforms and changes its color in ways that violate the laws of physics. \algname highlights the corresponding regions as areas of high uncertainty.
Further, our proposed method is able to capture uncertainty from occlusions, which is shown in~\Cref{fig:occlusion}. 
While grasping the red spoon, the robot's arm occludes some areas of the scene. As the task proceeds, \algname identifies the occluded areas behind the robot as high-uncertainty regions, in line with intuition.

%% file: sections/evaluation_additional.tex
\subsection{Detecting OOD Inputs at Inference}
\label{sec:ood_experiments}

Here, we explore the performance of \algname in out-of-distribution (OOD) detection at inference time, noting the importance of calibrated uncertainty estimates in reliable OOD detection.
Concretely, a trustworthy video model would express higher uncertainty when given a task that lies outside of the training distribution, reflecting its lack of knowledge of the scene and object dynamics. 
We examine the calibration of the uncertainty estimates computed by our method in these settings through real-world experiments on a WidowX 250 robot in the Bridge setup within a toy kitchen environment.

We consider OOD conditions across five axes: background, lighting, environment clutter, target object (task), and robot end-effector, creating environment settings that are noticeably different from those seen in the Bridge dataset. 
For the \emph{background} axis, we introduce novel background objects into the scene, e.g., computer accessories and sport equipment. For \emph{lighting}, we vary the RGB value of the environment lighting. We add more objects to the scene in the \textit{environment clutter} setting and introduce novel target objects or objects in unseen configurations for grasping in the \textit{target object} test setup. Along the \textit{end-effector} axis, we create OOD conditions by modifying the appearance of the end-effector by attaching lightweight objects (e.g., a towel, plushy toys, etc.) to the robot without noticeably altering the robot dynamics.
In these settings, we collect $50$ ground-truth trajectories ($10$ trajectories per setting) and generate videos from the video model using the associated robot actions.

In~\Cref{fig:experiments_ood_frames}, 
we show the ground-truth and generated videos and composited uncertainty maps for one trajectory in the \emph{Background} and \emph{Lighting} categories, with additional visualizations provided in Appendix~\ref{app:ood_experiments}. For unfamiliar background objects (e.g., a skeleton), we observe that the model becomes uncertain about the dynamics between the robot and the background object as it approaches the background object, which can be seen in the generated video. \algname localizes this uncertainty, accurately delineating more confident video patches from less confident ones.
Likewise, we see that the video model struggles to generate accurate videos under unseen lighting conditions, with an observable degradation in the video quality over time. Our method again captures the increasing uncertainty of the video model spatio-temporally.

\begin{figure}[th]
    \centering
    \adjustbox{width=\linewidth}{
        \includegraphics[width=\linewidth]{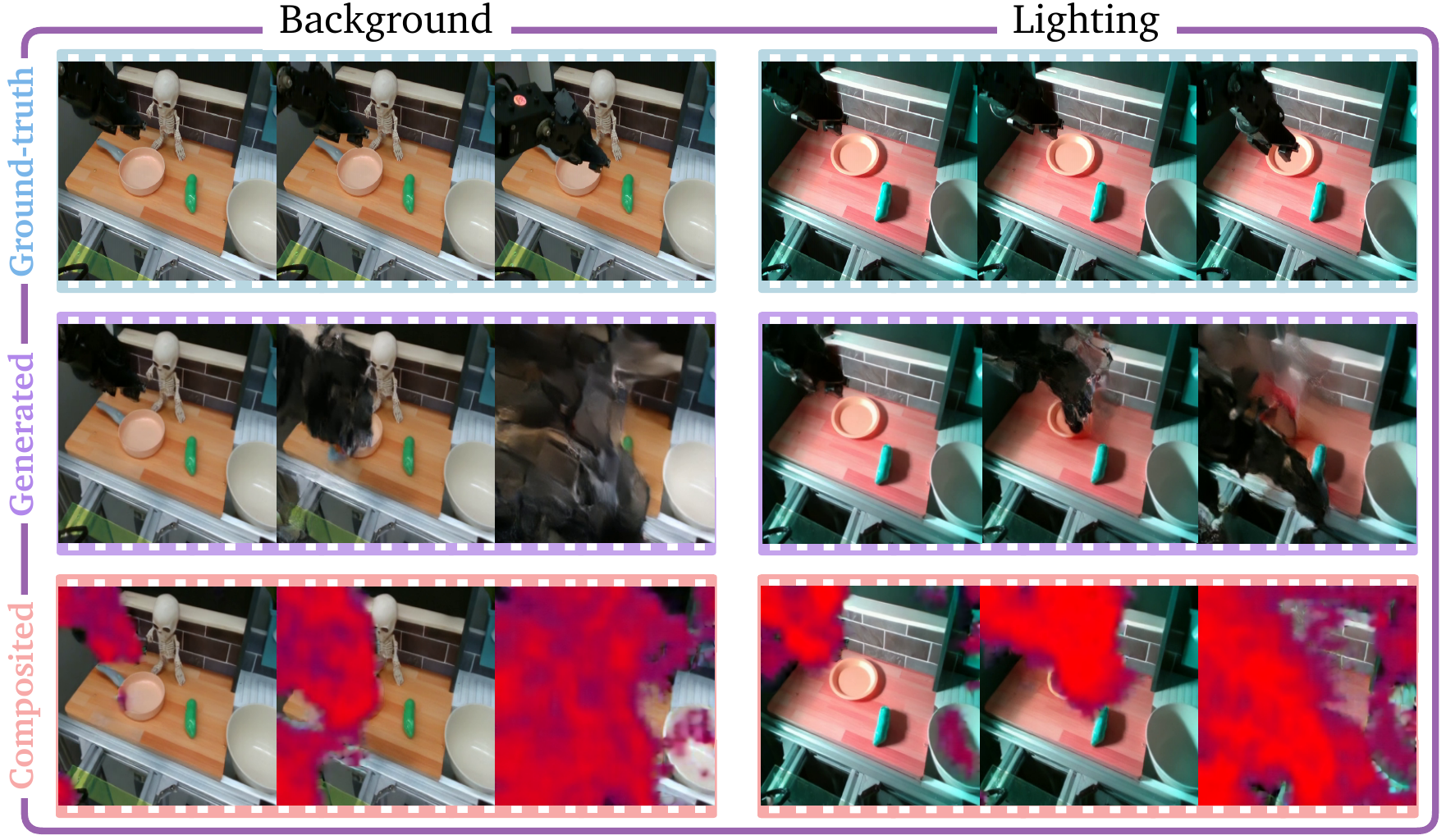}
    }
    \caption{\textbf{OOD detection.} \algname is able to accurately localize hallucinations (identified by the red regions) in OOD scenarios where the model is presented with inputs outside of its training distribution.}
    \label{fig:experiments_ood_frames}
\end{figure}

\subsection{Ablations}
\label{sec:exp_ablations}
We conduct ablation studies on the effects of simultaneous video generation and uncertainty quantification, alignment between latent space error and video quality metrics, and end-to-end training without stop-gradient. We discuss additional ablations in Appendix~\ref{app:ablations}.

\begin{table}[t]
  \centering
  \begin{minipage}{0.45\textwidth}
    \centering
    \caption{Perceptual Quality of the Vanilla Video Model without UQ}
    \label{tab:video_quality}
    \centering
    \setlength{\tabcolsep}{8pt}
        \begin{tabular}{l c c}
            \toprule
             Metric & Vanilla & \algname (\textbf{Ours}) \\
             \midrule
             SSIM & 0.75 & \textbf{0.76} \\
             PSNR & 18.4 & \textbf{18.6} \\
             LPIPS & 0.28 & 0.28 \\
             \bottomrule
         \end{tabular}
  \end{minipage}
  \hfill %
  \begin{minipage}{0.45\textwidth}
    \centering
    \caption{Comparison of Video Latent Error to Perceptual Metrics.}
    \label{tab:latent_error_to_perceptual}
        \begin{tabular}{l c c}
            \toprule
             Method & Correlation & Significance $(\uparrow)$ \\
             \midrule
             SSIM & $-0.54$ & $99\%$ \\
             PSNR & $0.83$ & $99\%$ \\
             LPIPS & $0.73$ & $99\%$ \\
             \bottomrule
         \end{tabular}
  \end{minipage}
  
\end{table}

\p{Effects on Video Generation Quality}
We report standard video quality metrics (SSIM, PSNR, LPIPS) 
for our UQ video model (\algname) compared to the vanilla video model without UQ in~\Cref{tab:video_quality}. Our model achieves marginally better scores on the perceptual metrics, showing that our method does not degrade video quality. %
In head-to-head rankings, our model outperforms the vanilla model with scores of 62.5\%, 58.8\%, and 62.5\% on SSIM, PSNR, and LPIPS.

\p{Validity between Latent Space Error and Video Quality Metrics}
\Cref{tab:latent_error_to_perceptual} shows the correlation between the $L_{1}$ latent space error and standard perceptual metrics. We find that the $L_{1}$ latent error is strongly correlated with SSIM, PSNR, and LPIPS at the $99\%$ significance level. Note that one would expect a negative correlation between the $L_{1}$ latent error and SSIM and PSNR, and a positive one between the $L_{1}$ latent error and LPIPS. Our results underscore that the $L_{1}$ latent error is well-aligned with standard perceptual metrics, without the additional cost of decoding the video latents and computing these perceptual metrics.

\p{End-to-End Training without Stop-Gradient}
We examine the calibration of \algname when training the video model and the UQ probe $\mbf{f}_{\phi}$ end-to-end without a stop-gradient operator between both models. In other words, we update the parameters of the video model with the gradient of $\mbf{f}_{\phi}$, assessing the existence of any training synergies from joint training. We train the \tc model with and without the stop-gradient operator and find no significant difference in the calibration of the resulting confidence estimates. We observe a difference of about $5\mathrm{e}^{-3}$ and $3\mathrm{e}^{-3}$ in the ECE and MCE, respectively, highlighting that both variants achieve the same level of calibration. However, we note that backpropagating the gradient of $\mbf{f}_{\phi}$ through the video model incurs additional computation overhead. As a result, the stop-gradient operation provides a computational edge, especially in large video models with billions of parameters.

%% file: sections/conclusion.tex
\section{Conclusion}
\label{sec:conclusion}
We present a method for calibrated controllable video synthesis that trains video models to know when they don't know. We use proper scoring rules as loss functions to achieve both accuracy and calibration in video generation. By quantifying uncertainty in the latent space, our proposed method overcomes computation challenges and training instability associated with pixel-space approaches. Furthermore, we map latent-space uncertainty to interpretable pixel-space confidence estimates that are well-aligned with human intuition. We show that our method is able to precisely localize hallucinations in generated videos. Likewise, we demonstrate the calibration of the confidence estimates computed by our method across different robot embodiments and tasks, and further show the effectiveness of \algname in detecting out-of-distribution inputs at inference time.

%% file: sections/limitations_future_work.tex
\section{Limitations and Future Work}
\label{sec:limitations_future_work}
Although we demonstrate the calibration of \algname in OOD scenarios, its theoretical calibration guarantees only hold within the training data distribution through the use of proper scoring rules. Specifically, the diversity of the training data and the presence of a distribution shift at inference influences the observed calibration of the confidence estimates. We emphasize that this limitation is not unique to our approach. Nevertheless, we reiterate that our results show that our method produces calibrated uncertainty estimates even in OOD settings.
Future work will explore training strategies for better coverage of the test distribution.
Additionally, long-duration temporal consistency of the confidence estimates computed by our method is limited by the history length of the conditioning inputs of the video model. With smaller historical contexts, our method may lose track of uncertain video patches over time. Long-duration video generation remains an open research problem, which will be explored in future work.
Looking forward, we believe that as applications of video models in robotics mature, rigorous uncertainty quantification will be crucial to their practical effectiveness.

%% file: sections/acknowledgments.tex
\section*{Acknowledgments}
The authors were partially supported by the NSF CAREER Award \#2044149, the Office of Naval Research (N00014-23-1-2148), and a Sloan Fellowship.

%% file: sections/appendix/appendix.tex
\input{sections/appendix/app_evaluations}

\input{sections/appendix/app_prelims_proof}

%% file: sections/appendix/app_evaluations.tex
\vspace{-1.5ex}
\section{Detecting OOD Inputs at Inference}
\label{app:ood_experiments}

\begin{figure}[t]
    \centering
    \adjustbox{width=\linewidth}{
        \includegraphics[width=\linewidth]{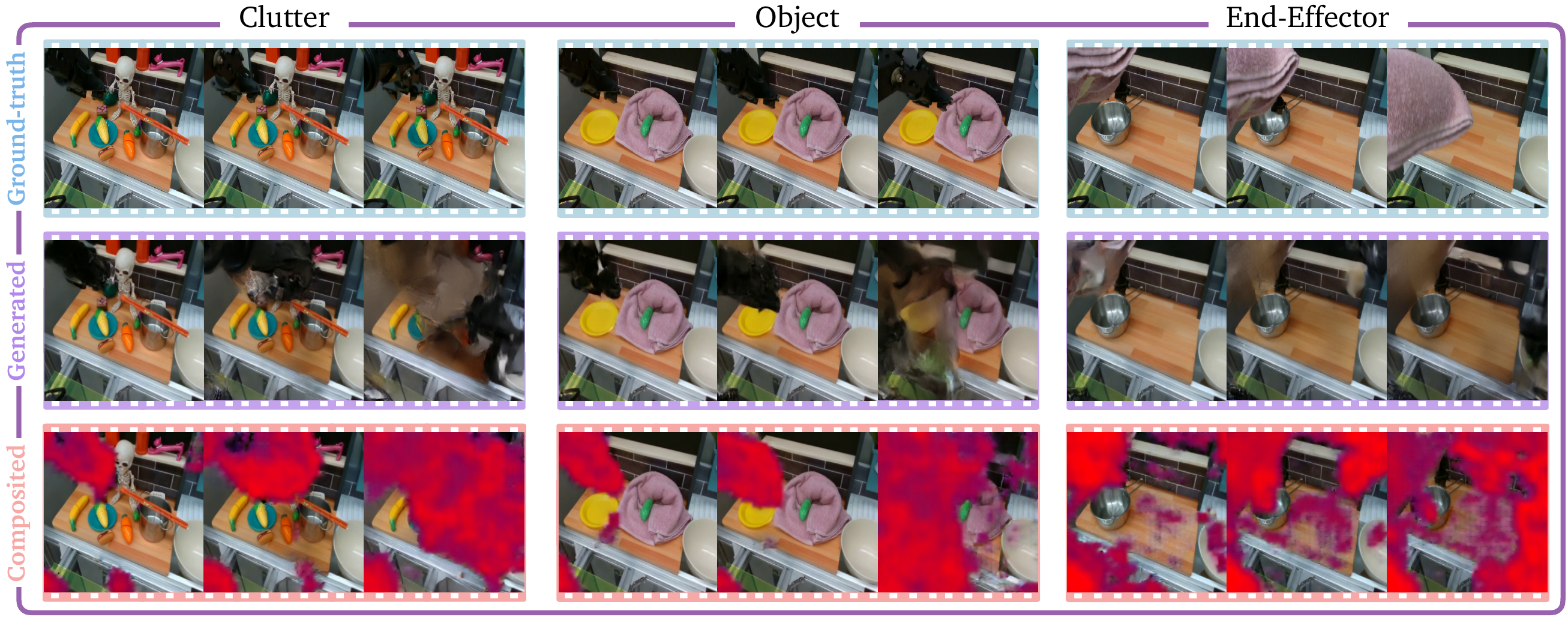}
    }
    \caption{\textbf{OOD detection}, \algname is able to accurately localize hallucinations (identified by the red regions) in OOD scenarios where the model is presented with inputs outside of its training distribution.}
    \label{fig:experiments_ood_frames_clutter_task_robot}
    \vspace{-2.5ex}
\end{figure}

\begin{wrapfigure}[29]{r}{0.35\textwidth}
    \vspace{-8ex}
    \centering
    \begin{minipage}{\linewidth}
        \includegraphics[width=\linewidth]{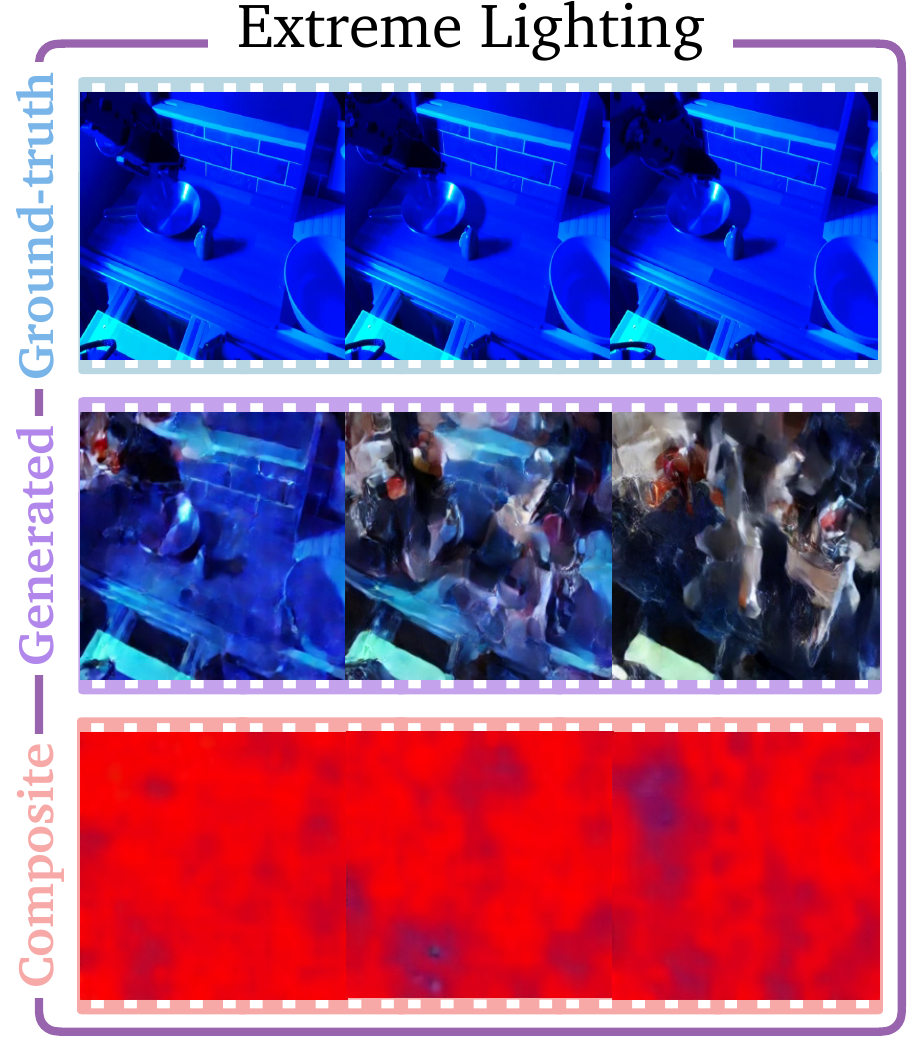}
        \caption{\textbf{Extreme Lighting.} \algname identifies hallucinations where the model tries to reset the lighting. }
        \label{fig:experiments_ood_lighting}
    \end{minipage}
    
    \vspace{1ex}
    
    \begin{minipage}{\linewidth}
        \includegraphics[width=\linewidth]{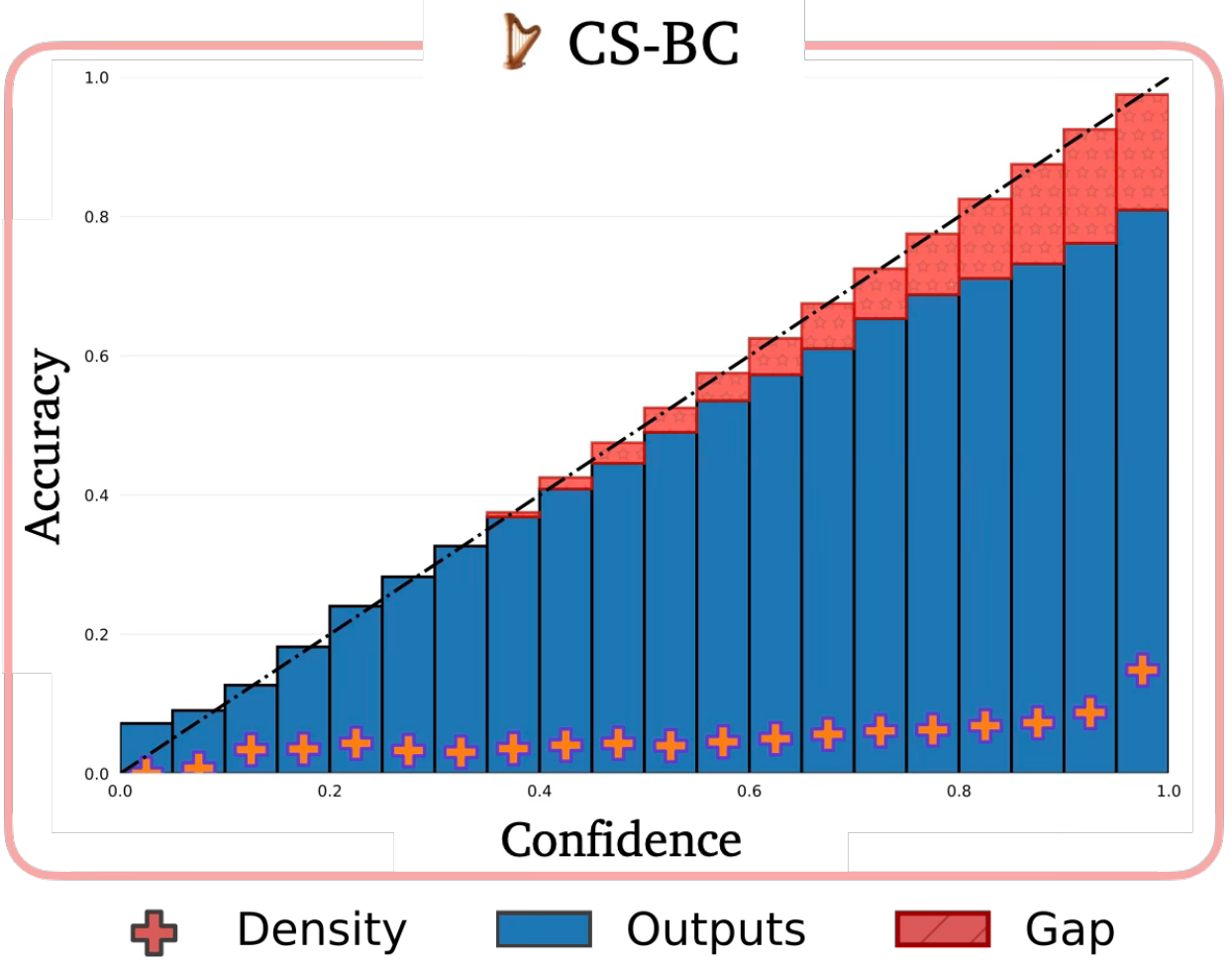}
        \caption{\textbf{Reliability diagram}, showing strong calibration.}
        \label{fig:experiments_ood_calibration}
    \end{minipage}%
\end{wrapfigure}
We provide additional visualizations of the ground-truth and generated videos and confidence predictions under different OOD conditions in~\Cref{fig:experiments_ood_frames_clutter_task_robot,fig:experiments_ood_lighting}.
In cluttered environments, the video model fails to accurately predict the interaction between the robot and the objects in the scene, which \algname correctly identifies. 
Lastly, when an unfamiliar object is attached to the robot end-effector, the video model becomes uncertain about the dynamics of the robot, leading to hallucinations in the generated video. Our method identifies these hallucinations as regions of high uncertainty.
Under extreme lighting conditions, the video model hallucinates a recoloring of the scene in an attempt to match the training data distribution. \algname reveals the model's uncertainty, specifically identifying regions with edited colors as areas of high uncertainty.
In~\Cref{fig:experiments_ood_calibration}, we provide the reliability diagram of \algname in OOD environments. We observe that our method remains well-calibrated with only a very small drop in calibration compared to its performance under nominal conditions. As stated in~\Cref{sec:exp_interpretable}, \algname achieves low calibration errors, with an ECE and MCE of~${9.98\mathrm{e}^{-2}}$ and ${1.71\mathrm{e}^{-1}}$, respectively.

\vspace{-1.5ex}
\section{Calibration on the Bridge Dataset}
\label{app:interpretable_experiments_bridge}
We assess underconfidence vs. overconfidence of our proposed UQ method using reliability diagrams, which visualize the calibration error associated with the uncertainty estimates across different confidence bins. 
In~\Cref{fig:average_reliability}, we show the reliability diagram for each model averaged across all thresholds.
The dashed line in each plot traces the path of perfect calibration, while the cross-shaped markers indicate the density of samples in each bin. Across all models, we observe that with \algname, the video models are well-calibrated, i.e., neither underconfident nor overconfident. Notably, the models tend to be more conservative when unsure about the accuracy of the generated video, as visualized by the bars in the ${[0.3, 0.7]}$ confidence bins exceeding the dashed line. This emergent behavior aligns well with trustworthiness in safety-critical applications, with a greater propensity for the model to express doubt when unsure about the accuracy of the generated video.
Additionally, note that FSC provides the best calibration but lacks flexibility, since it relies on a fixed error threshold. In contrast, MCC and CSBC are more general architectures that trade off \emph{marginal} levels of calibration for broader expressiveness. 
Further, in~\Cref{fig:reliability_fixed_scale_compare}, we compare calibration of the \tc and \nc models at the fixed scale (${\varepsilon_{v} = 0.5}$) used in training the \nc model. We find that both models are well-calibrated, producing relatively the same reliability diagrams.

In~\Cref{fig:calibration_each_threshold}, we provide additional results for \tc, showing calibration across different accuracy threshold levels.
Overall, we see that \algname is well-calibrated across each error threshold, with the top of the confidence bins tracing the diagonal line. 
Further, we observe greater uncertainty at lower thresholds (e.g., ${\varepsilon_v = 0.2}$) which aligns with the intuition that lower accuracy thresholds are generally associated with greater uncertainty in the accuracy of the prediction given the tightness of the threshold.
Conversely, as the threshold increases, the sample densities gradually shift right toward the higher confidence region, aligning with the intuition that larger thresholds afford greater confidence in the accuracy of the generated videos.
Moreover, we observe that at extremely low values of $\varepsilon_v$ ($\varepsilon_v \leq 0.3$), \algname tends to be underconfident, signified by the histograms going above the line of perfect calibration. Further, note that the model's degree of underconfidence decreases as the accuracy threshold increases. Overall, our observations are well-aligned with safety. The video model tends to be more conservative at very low accuracy thresholds, mitigating false negatives, i.e., inaccurate patches that are identified as highly confident regions, which could otherwise lead to harmful consequences.

\begin{figure*}[t]
    \centering
    \includegraphics[width=0.8\linewidth]{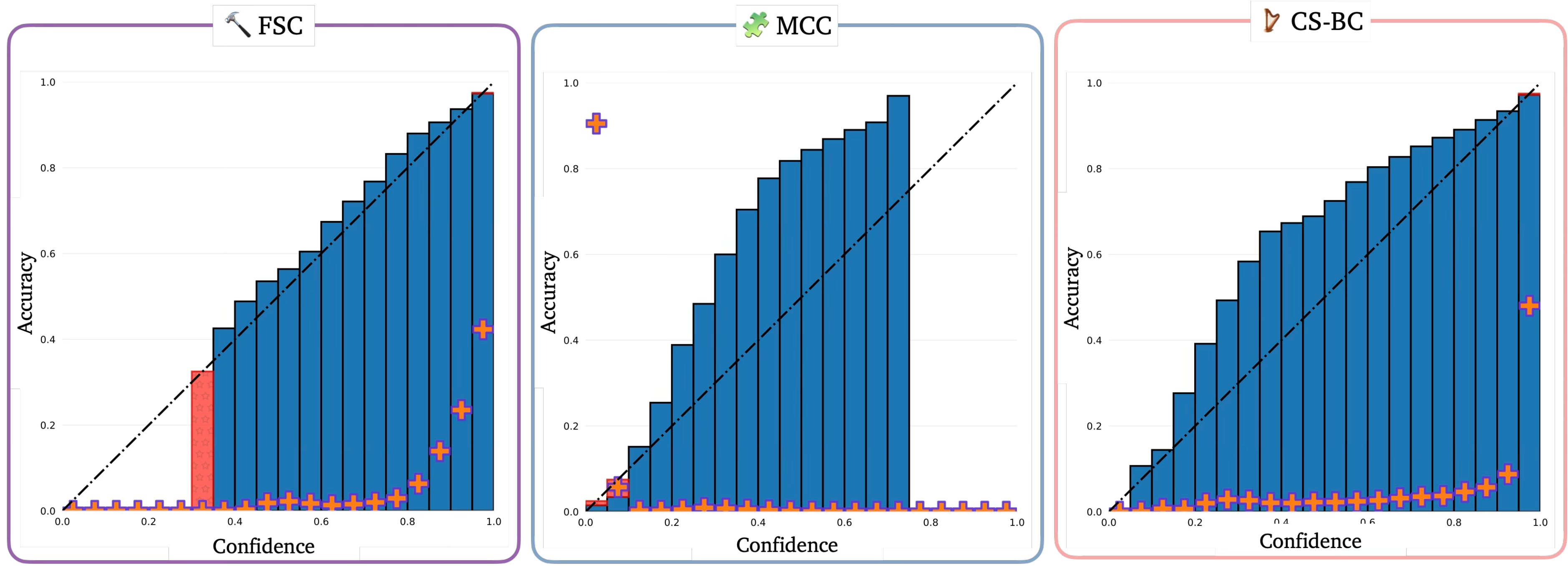}
    \caption{\textbf{Aggregated reliability diagrams.} All methods are well-calibrated.
    Note that FSC provides the best calibration but lacks flexibility, given its dependence on a fixed error threshold. In contrast, MCC and CSBC are more general models that trade off marginal calibration for broader effectiveness.
    }
    \label{fig:average_reliability}
\end{figure*}

\begin{figure}[t]
    \centering
    \includegraphics[width=0.65\linewidth]{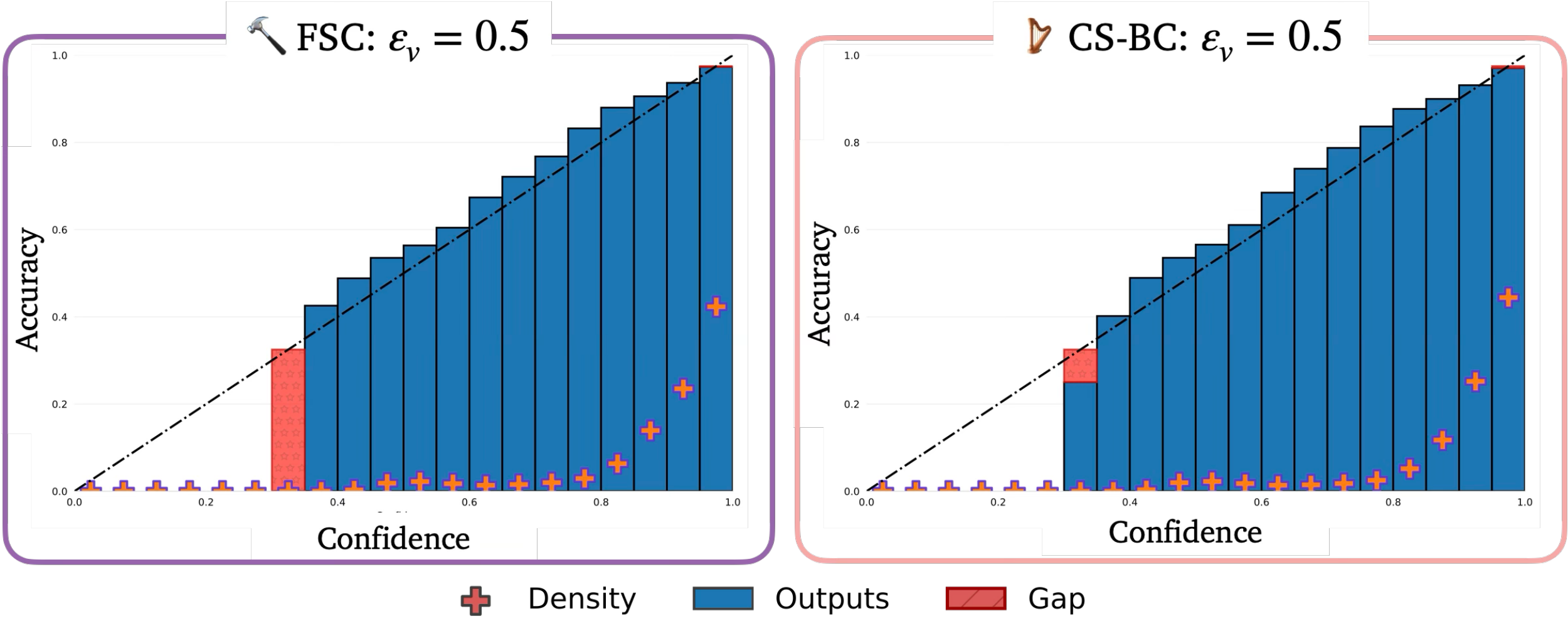}
    \caption{\textbf{Reliability diagrams for FSC and CS-BC.} 
    Both models achieve similar calibration at threshold ${\varepsilon_v = 0.5}$.
    }
    \label{fig:reliability_fixed_scale_compare}
\end{figure}

\begin{figure*}[t]
    \centering
    \includegraphics[width=\linewidth]{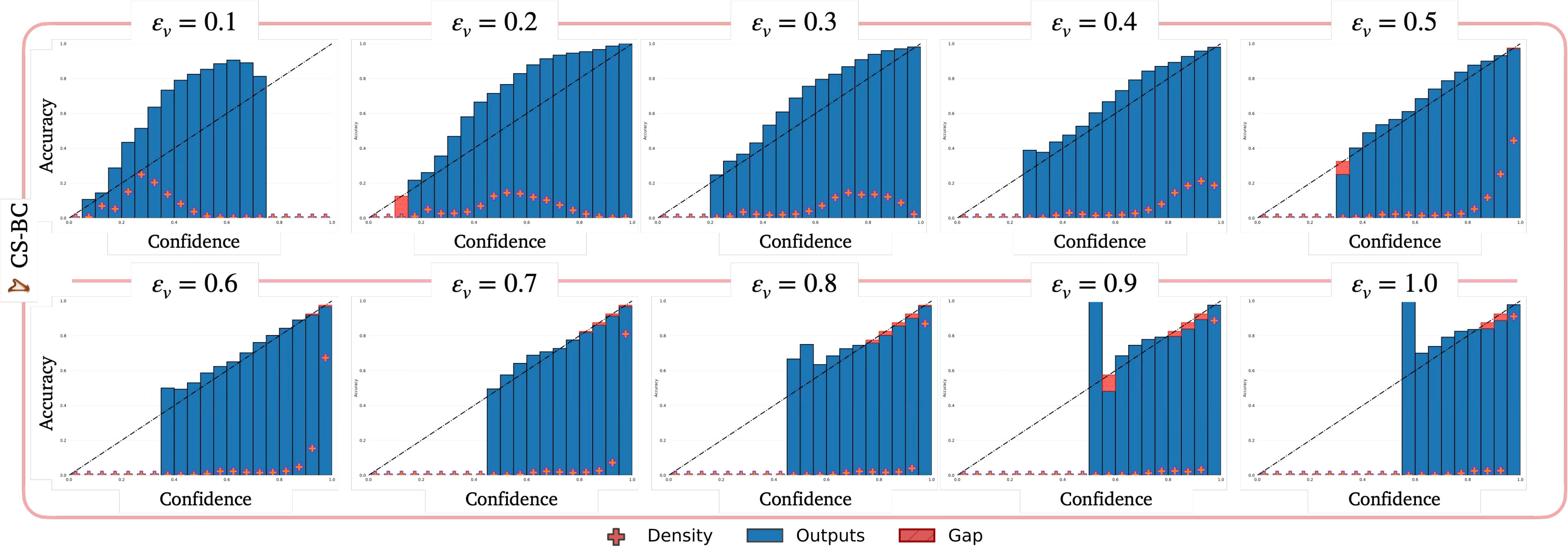}
    \caption{\textbf{Reliability diagram for each threshold.} \algname is well-calibrated across all accuracy thresholds, with some degree of conservativeness at very low thresholds.}
    \label{fig:calibration_each_threshold}
\end{figure*}

\section{Evaluations on the DROID Dataset}
\label{app:droid}
We conduct additional experiments verifying the effectiveness of \algname on the DROID dataset~\cite{khazatsky2024droid}, which covers a much wider range of tasks and environments compared to the Bridge dataset. %
We train the \tc model on this dataset and evaluate the calibration and interpretability of its confidence estimates.
We compute the calibration errors across the test dataset with ECE and MCE values of $7.28\mathrm{e}^{-2}$ and $1.74\mathrm{e}^{-1}$, respectively.
Note that the ECE is again close to the lower bound of the range of calibration errors, highlighting the well-calibrated nature of \algname.
In summary, we find that our method is well-calibrated across a wide range of video prediction robotics problems, with broad amenability to multi-view camera inputs in diverse environments. 
Moreover, the results indicate that \algname remains effective across different robot embodiments, despite the notable difference between the Panda robot in the DROID dataset and the WidowX robot in the Bridge dataset.

\p{Interpretability}
Next, we examine the interpretability of the confidence estimates computed by \algname.
First, we compute the correlation between the predicted confidence and the absolute error between ground-truth and generated latent videos. Similar to the Bridge dataset, a negative correlation between both quantities indicates greater interpretability of the uncertainty estimates.
On the DROID dataset, we observe a negative correlation coefficient of $-0.149$ with a significance level greater than $99\%$, showing a desirable alignment between the model's estimated confidence and the observed accuracy of the generated videos. In other words, with \algname, the video model tends to be more uncertain about the generated videos when it's more likely wrong. Our results can be explained by the use of proper scoring rules to achieve both calibration and accuracy.

\begin{wrapfigure}[13]{r}{0.4\textwidth}
    \vspace{-4.2ex}
    \centering
    \includegraphics[width=0.9\linewidth]{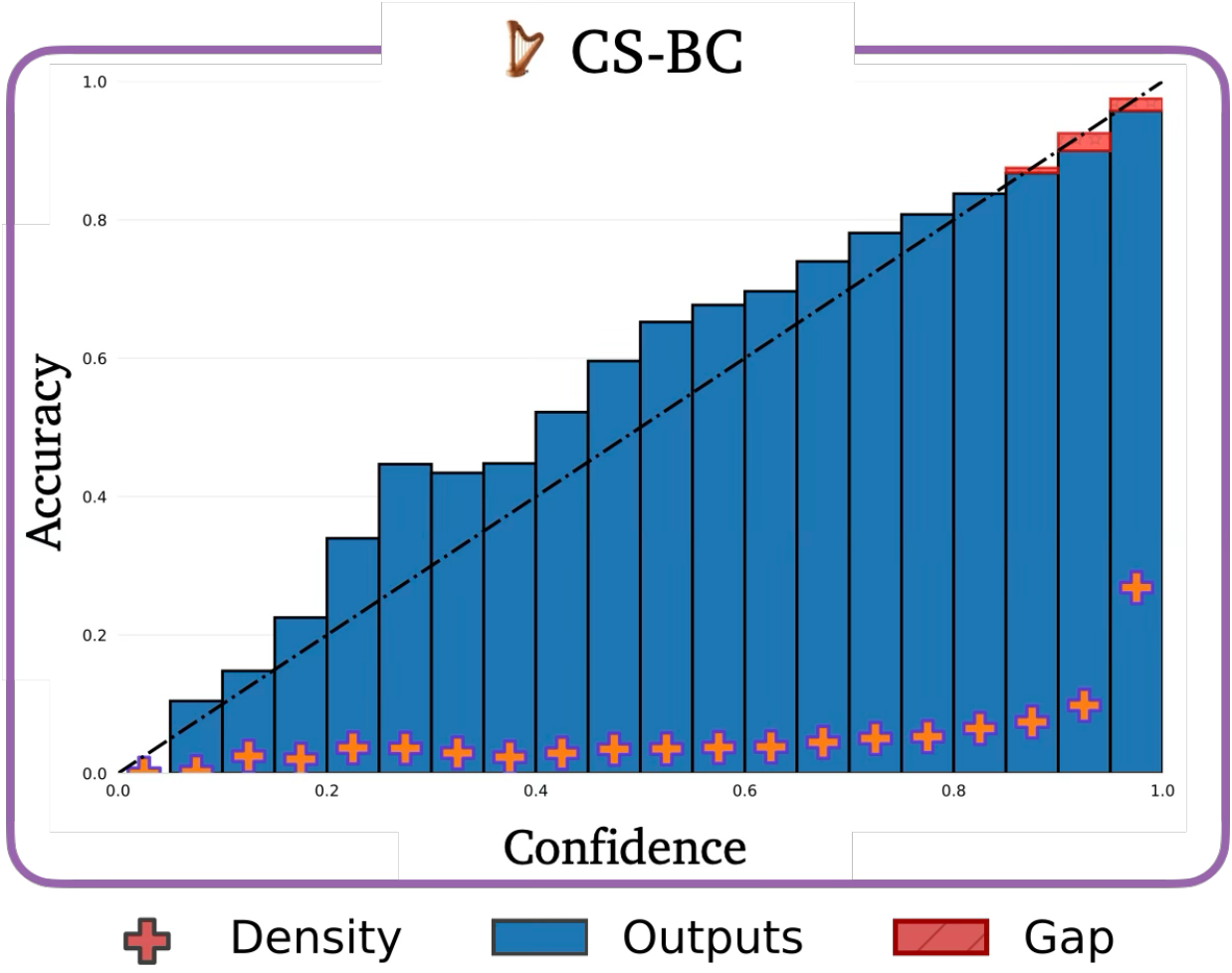}
    \caption{\textbf{Reliability diagram on DROID}, showing calibration.}
    \label{fig:droid_calibration}
\end{wrapfigure}
We provide visualizations of the ground-truth and generated videos, along with the estimated confidence, highlighting the calibration and interpretability of \algname's uncertainty predictions.
First, in~\Cref{fig:droid_calibration}, 
we show the reliability diagram of \algname computed across the videos in the test dataset. 
Similar to the Bridge dataset, we observe a near-perfect calibration of the confidence estimates.

\begin{figure}[h]
    \centering
    \adjustbox{width=0.7\linewidth}{
        \includegraphics[width=\linewidth]{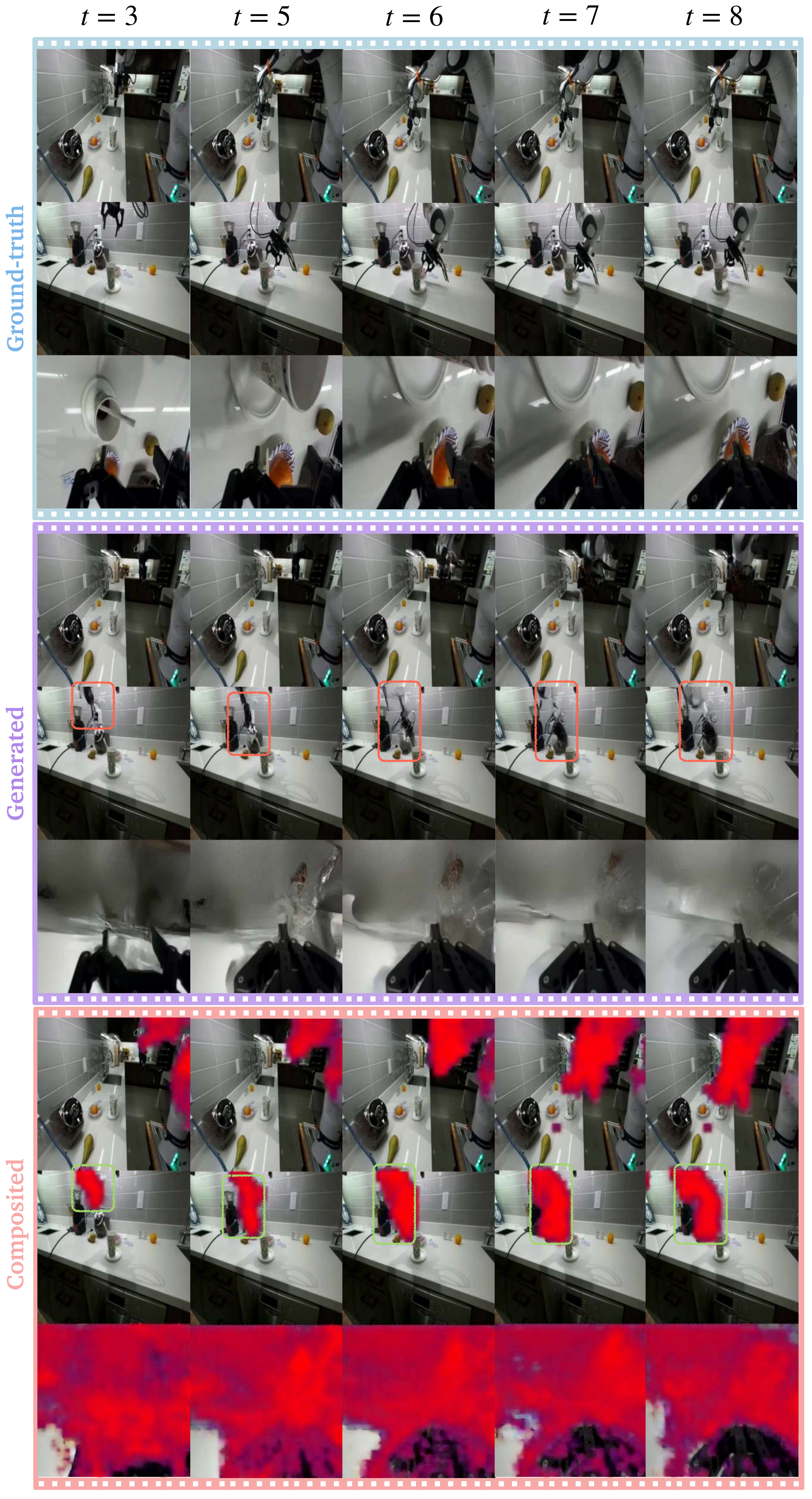}
    }
    \caption{\textbf{Hallucination.} \algname identifies hallucinations in the generated videos from the DROID dataset as areas of high uncertainty.}
    \label{fig:droid_hallucination}
\end{figure}

From~\Cref{fig:droid_hallucination}, we observe that the video model is able to successfully generate multi-view videos, capturing the evolution of the task from two side-camera views (shown in the first-two rows) in addition to a wrist-camera view (shown in the third row).
However, we also see a degradation in the relative quality of the generated videos compared to the Bridge dataset, reflecting the increased difficulty associated with multi-view video generation compounded with the greater diversity of the DROID dataset.
This observation is particularly conspicuous in the wrist-camera view generated by the video model where the details of the scene quickly fade into a blurry background.
Despite this degradation in video quality, \algname is still able to produce interpretable, calibrated uncertainty estimates at a fine-grained level, localizing non-confident regions of the video in each camera view.
In particular, we see that in the right-camera view of the generated video (second row), our method captures hallucinations of the robot's gripper that appear in the video---the gripper morphs and elongates.
Likewise, \algname correctly identifies the inaccurate blurry background in the wrist-camera view (third row) as a region of high uncertainty.

\clearpage

\section{Additional Ablations}
\label{app:ablations}

We perform additional ablations to study the effects of different scoring rules, diffusion forcing, and end-to-end training without the stop-gradient operation, with respect to calibration of the confidence estimates computed by \algname. %

\p{Proper Scoring Rules} We examine the calibration of video models trained with the binary cross-entropy loss function and the Brier loss function. Since both loss functions are proper scoring rules, one would expect both models to achieve similar calibration performance. 
We train two variants of the \tc model using these loss functions and observe similar calibration levels between both models, in line with the preceding expectation. The results indicate a negligible difference in the ECE of about ${3\mathrm{e}^{-4}}$ and a similarly negligible difference in the MCE of about ${6\mathrm{e}^{-3}}$.

Further, we visualize the reliability diagrams associated with each model in~\Cref{fig:ablation_scoring_rule_reliability}. We see very similar trends in the reliability diagrams across all confidence bins. Specifically, the models are well-calibrated but more conservative when unsure about the accuracy of the generated video. These findings underscore the general applicability of our proposed framework to strictly proper scoring rules.

\begin{figure}[th]
    \centering
    \includegraphics[width=0.7\linewidth]{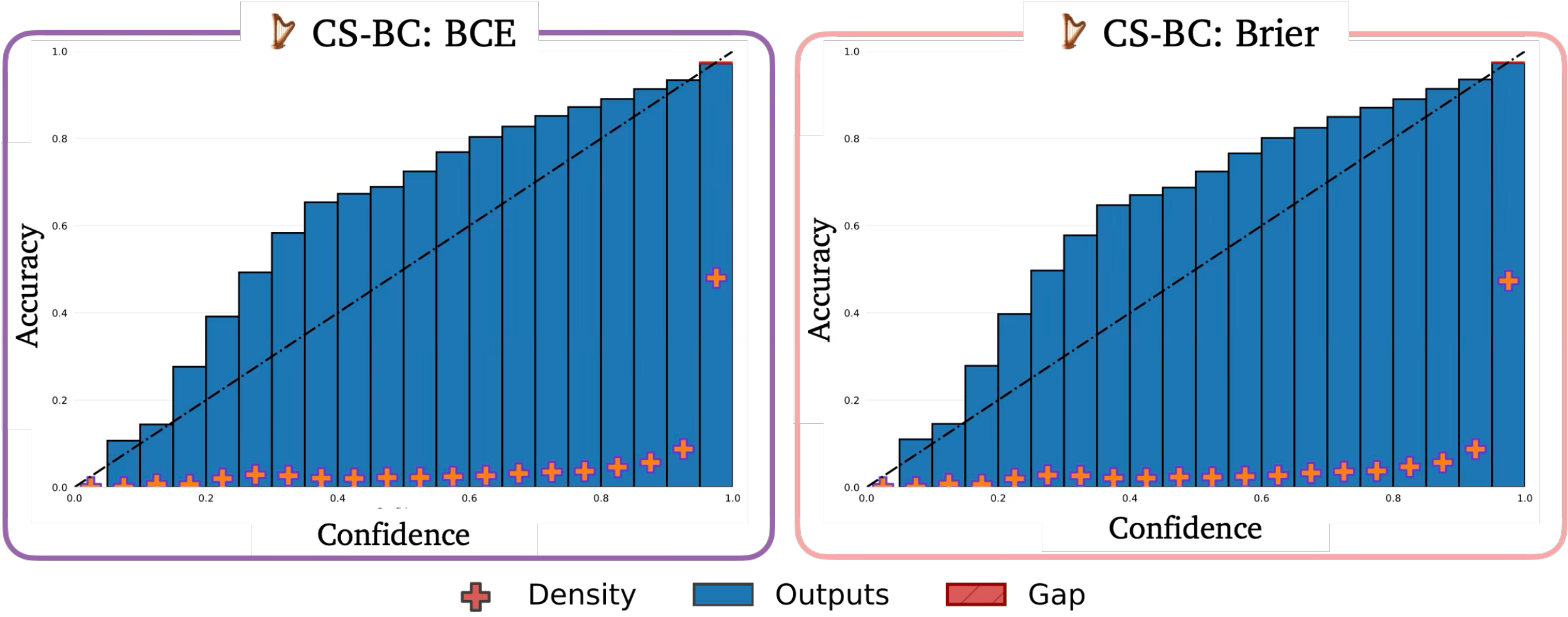}
    \caption{\textbf{Reliability diagram for ablation on proper scoring rules.} \algname remains well-calibrated with the BCE and Brier scores.}
    \label{fig:ablation_scoring_rule_reliability}
\end{figure}

\p{Diffusion forcing} 
We explore diffusion forcing in generating the confidence maps during video prediction and compute the ECE and MCE to evaluate calibration. We visualize the reliability diagrams of the \tc model with and without diffusion forcing in~\Cref{fig:ablation_diffusion_forcing}, showing that diffusion forcing degrades the calibration of \algname. We hypothesize that this observation could be due to the effects of the recurrence in diffusion forcing, which leads to a notable increase in the conservatism of the confidence estimates. With diffusion forcing, the ECE rises to about $3.3\mathrm{e}^{-1}$ with an associated increase in the MCE to about $5.54\mathrm{e}^{-1}$. A comprehensive analysis of the effects of diffusion forcing on confidence prediction lies beyond the scope of this paper; hence, we leave it to future work.

\begin{figure}[th]
    \centering
    \includegraphics[width=0.7\linewidth]{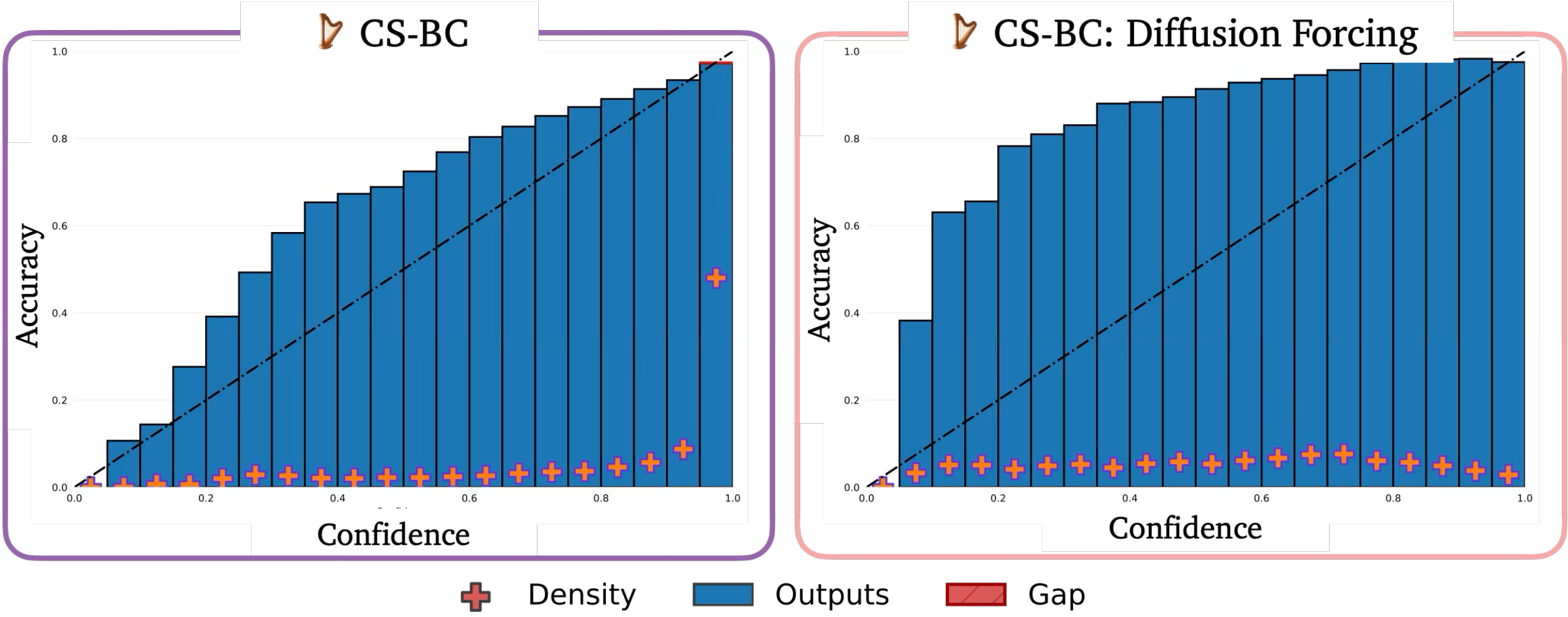}
    \caption{\textbf{Reliability diagram for ablation on diffusion forcing.} Diffusion forcing increases underconfidence of \algname, degrading calibration.}
    \label{fig:ablation_diffusion_forcing}
\end{figure}

\begin{wraptable}[8]{r}{0.5\textwidth}
    \vspace{-0.25ex}
    \centering
    \caption{Comparison to other UQ Baselines ($^{\star}$Ensembles are broadly impractical due to high cost.)}
    \label{tab:baselines}
        \begin{tabular}{l c c}
            \toprule
             Method & Correlation & Significance $(\uparrow)$ \\
             \midrule
             \algname~(\textbf{ours}) & $-0.36$ & $99\%$ \\
             Ensemble$^{\star}$ & $0.42$ & $99\%$ \\
             Heuristic & $-0.01$ & $43\%$ \\
             \bottomrule
         \end{tabular}
\end{wraptable}
\p{Additional Baselines}
Established UQ baselines are prohibitively expensive to implement for large generative models, e.g., deep ensembles would require significantly greater GPU VRAM and compute time.
Nonetheless, we compare our method to an approximate implementation of deep ensembles, where we exploit the video model stochasticity for multiple generations, using the resulting variance among generated videos as an uncertainty signal.
Further, we compare against a heuristic baseline using the raw diffusion noise as uncertainty. 
Since these baselines do not estimate uncertainty with a  
probability distribution, standard calibration metrics for classification (e.g., ECE/MCE) do not apply.
Consequently, we examine the correlation between latent space error and estimated uncertainty to evaluate calibration, which is standard for these estimates.
\Cref{tab:baselines} shows their correlation coefficients and significance levels $\alpha$. \algname measures confidence and should ideally be negatively correlated with error, while ensemble and heuristics measure uncertainty and thus should be positively correlated. \algname outperforms the heuristic baseline (whose uncertainty estimates are not statistically significant) and is more competitive than the ensemble, given its lower computation cost while achieving comparable calibration.

\section{Additional Implementation Details}
\label{app:eval_procedure}

\p{Model Training and Evaluation}
We implement the video generation model using a diffusion transformer architecture with $49$ transformer layers, each with $4$ heads and an embedding dimension of $512$. We use the Stable Video Diffusion (SVD) VAE~\cite{blattmann2023stable} for encoding the videos into the latent space, extracting video patches with no temporal compression. Note that we use SVD for its generality although our approach is amenable to other (larger) models.
Using a learning rate of $1\ex^{-5}$ with a cosine decay scheduler, we train the video model for $50$k iterations with a batch size of $4$ for the Bridge dataset and a batch size of $2$ for the DROID dataset, on 8 NVIDIA L40 GPUs. We use an input video resolution of $256 \times 256$ for both datasets. We stack the multi-view camera inputs in the DROID dataset to construct a single video frame as input.

We train on the entire train split of the Bridge dataset; however, given the large size of the DROID dataset, we only train on a subset (TRI), covering both success and failure videos, across a broad range of tasks and environments.
When training the \tc model, we randomly sample thresholds from a discrete set of $28$ threshold values constructed linearly from 0.1 to 1 with adaptive (denser) spacing at lower thresholds between $0.1$ and $0.3$ to more effectively capture the fine-grained signal existent in this subrange. We define the output bins of the \mc model using the same set of threshold values. 
We evaluate the video models on $110$ and $83$ trajectories in the test split of the Bridge and DROID datasets, respectively.

When evaluating the \tc model at inference time, we use $10$ linearly-spaced values of $\varepsilon_v$ ranging from $0.1$ to $1.0$.
Except otherwise noted in our qualitative evaluations, we visualize the confidence predictions at the midpoint of the range corresponding to threshold values of $0.5$ or $0.6$. We compute the aggregate calibration errors for each model using standard implementations over $20$ bins and compute the video accuracy using the $\ell_{1}$ loss.

\p{Visualizing Latent-Space Errors}
We visualize the latent-space error between representative generated videos and the corresponding ground-truth videos in the RGB space in~\Cref{fig:latent_space_error}, to aid understanding of the accuracy resolutions discussed in this work. 
Intuitively, one would expect a calibrated video model to more confidently identify regions with very low or high errors as accurate or inaccurate, respectively, and to be more uncertain about other regions.
For alignment with human intuition, we construct a latent-space color map with three basis colors, with the extreme points of the color map defined by \emph{blue} and \emph{green}, corresponding to the minimum point (low-error region) and the maximum point (high-error region) of the error span, respectively. We use the color \emph{red} to represent the middle region of the error span, which induces an interpretable confidence heatmap, discussed later in this section.  
Notably, \Cref{fig:latent_space_error} shows that an error span over the range $[0, 1]$ is sufficient to capture essentially all observable error values. Specifically, the resulting color maps contain almost no green region, associated with the maximum end of the error span. We find that most of the error values lie within the selected error span, further justifying the accuracy resolutions utilized in our experiments.

\begin{figure}[th]
    \centering
    \includegraphics[width= \linewidth]{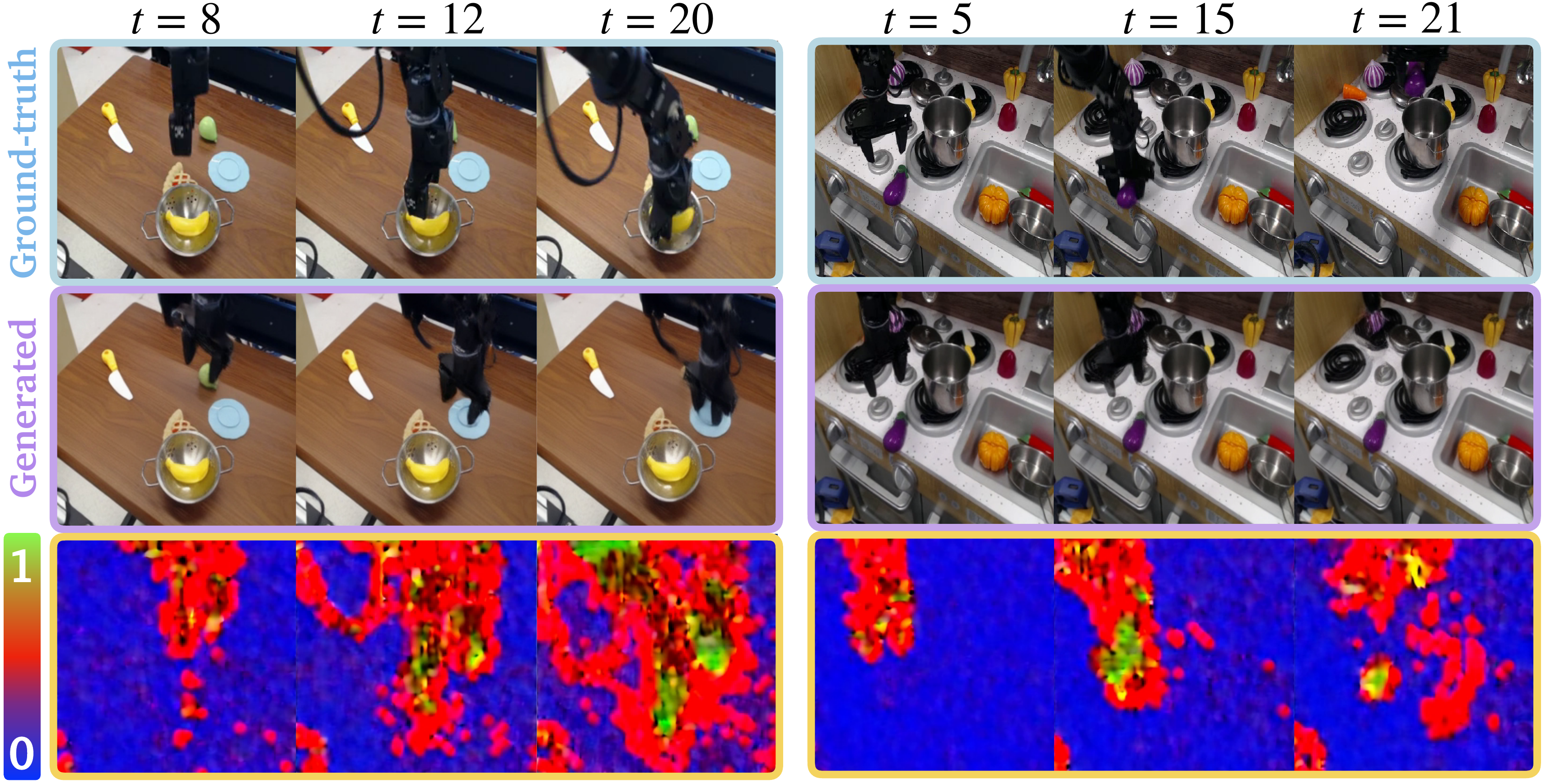}
    \caption{\textbf{Latent space error.} We visualize the latent-space video error in the RGB space, showing the observable range of the errors.}
    \label{fig:latent_space_error}
\end{figure}

%% file: sections/appendix/app_prelims_proof.tex
\section{Preliminaries}
We provide relevant background material on video models and proper scoring rules.

\subsection{Video Diffusion/Flow-Based Models}
\label{app:diffusion}

Video diffusion models (and more generally flow-based video generation models) \cite{peng2025open, wan2025wan, agarwal2025cosmos, kong2024hunyuanvideo, yang2024cogvideox} have emerged as the dominant model architecture for controllable, high-fidelity video generation, capturing fine-grained scene detail and longer video durations compared to alternative architectures. 
Video diffusion models learn a data distribution $p_{\theta}(\mbf{x})$ over video samples ${\mbf{x} \in \mcal{U}}$ by first destroying the underlying structure in the training data through a \emph{forward} diffusion process and subsequently restoring the structure through a \emph{reverse} diffusion process \cite{sohl2015deep, ho2020denoising}, where $\mcal{U}$ represents the space of videos and $\mbf{x}$ consists of a sequence of video frames.

In denoising diffusion probabilistic models (DDPM)~\cite{ho2020denoising}, the forward diffusion process adds Gaussian noise to the training data following a Markov chain, while the reverse diffusion process recovers the target data by denoising pure noise ${\mbf{x}_{T} \sim \mcal{N}(0, \mbf{I})}$ through the procedure:
\begin{equation}
    \label{eq:closed_form_forward_diffusion_process}
    \begin{aligned}
        q(\mbf{x}_{t} \mid \mbf{x}_{0}) \sim \mcal{N}(\mbf{x_{t}}; \sqrt{\bar{\alpha}_{t}} \mbf{x_{0}}, (1 - \bar{\alpha}_{t}) \mbf{I}),
    \end{aligned}
\end{equation}
where ${\alpha_{t} \coloneqq 1 - \beta_{t}}$ and ${\bar{\alpha}_{t} \coloneqq \prod_{s = 1}^{t} \alpha_{s}}$, with ${\Sigma_{\theta}(\mbf{x}_{t}, t) = \beta_{t} \bm{I}}$.
For more stable training, the learning problem is reformulated as a noise prediction problem, where the learned model predicts ${\bm{\epsilon}_{t} = (\sqrt{1 - \bar{\alpha}_{t}})^{-1} \mbf{x}_{t} - \sqrt{\bar{\alpha}_{t}} \bm{\mu}}$ and optimizes the loss function:
\begin{equation}
    \label{eq:video_diffusion_loss_fn}
    \mcal{L}_{\theta} = \Expect_{t, \bm{\epsilon}, \mbf{x}_{0}} \left[\norm{\bm{\epsilon} - \bm{\epsilon}_{\theta, t}(\mbf{x}_{t}, t)}_{2}^{2}\right],
\end{equation}
via stochastic gradient descent.

In practice, we use denoising diffusion implicit models (DDIMs) \cite{song2020denoising}, which generalize DDPMs to a non-Markovian forward diffusion process. In effect, DDIMs decouple the forward and reverse diffusion timesteps and leverage this feature to accelerate the video generation process during inference. Concretely, with DDIMs, we can train the diffusion model using longer forward diffusion timesteps and generate new videos using shorter reverse diffusion timesteps, significantly reducing the generation time and computation overhead. Rather than predicting the noise $\bm{\epsilon}$, we predict the velocity $\bm{v}$ with a diffusion transformer.

We train the video model using a robot dataset 
\begin{equation}
\mcal{D} = \{((I_{j, t}, I_{j, t + 1}, a_{j, t}),\ \forall t \in [T_{j}]),\ j = 1,\ldots,N\},
\end{equation}
consisting of $N$ trajectories. Each data sample consists of the current observation ${I_{j, t} \in \mbb{R}^{H \times W \times C}}$, the next observation ${I_{j, t + 1} \in \mbb{R}^{H \times W \times C}}$, and the corresponding action ${a_{j, t} \in \mbb{R}^{m}}$, for the $j$'th trajectory of length $T_{j}$.

At inference, we sample new video frames using: 
\begin{equation}
    \label{eq:diffusion_sampling_process}
    \begin{aligned}
        \mbf{x}_{t - 1} \coloneqq \sqrt{\bar{\alpha}_{t - 1}}\ \tilde{\mbf{x}}_{0}(t)  + \sqrt{1 - \bar{\alpha}_{t - 1}} \left(\frac{\mbf{x}_{t} - \sqrt{\bar{\alpha}_{t}} \tilde{\mbf{x}}_{0}(t)}{\sqrt{1 - \bar{\alpha}_{t}}} \right),
    \end{aligned}
\end{equation}
where ${\tilde{\mbf{x}}_{0}(t) = \sqrt{\bar{\alpha}_{t}}\ \mbf{x}_{t} - \sqrt{1 - \bar{\alpha}_{t}}\ \bm{v}_{\theta}(\mbf{x}_{t}, \mbf{c})}$ denotes the value of $\mbf{x}_{0}$ predicted at timestep $t$ and $\mbf{c}$ denotes the action and timestep embeddings.

\begin{remark}[Velocity-space Accuracy]
    The distance function in \Cref{eq:latent_video_distance} requires executing the reverse diffusion process to generate the latent video $\mbf{x}$, which is computationally expensive during training. To address this challenge, we express the distance function in terms of the predicted and ground-truth velocities, ${\mbf{v}_{\theta}}$ and ${\mbf{v}^{\star}}$, respectively. By manipulating~\Cref{eq:diffusion_sampling_process} algebraically, we derive the relation: 
    \begin{equation}
        \label{eq:distance_function_as_velocity}
        \mbf{d}(\mbf{x}, \mbf{x}^{\star}) = \big\lvert \sqrt{\bar{\alpha}_{t}(1 - \bar{\alpha}_{t - 1})} - \sqrt{\bar{\alpha}_{t - 1}(1 - \bar{\alpha}_{t})} \ \big\rvert\ \mbf{d}(\mbf{v}, \mbf{v}^{\star}),
    \end{equation}
    with the corresponding boolean function $\bm{\acc}$ given by:
    \begin{equation}
        \label{eq:latent_video_accuracy_as_velocity}
        \bm{\acc}(\mbf{v}, \mbf{v}^{\star}) \coloneqq \mbf{d}(\mbf{v}, \mbf{v}^{\star}) \leq \varepsilon_{v},
\end{equation}
where ${\varepsilon_{v} = \frac{\varepsilon}{\big\lvert \sqrt{\bar{\alpha}_{t}(1 - \bar{\alpha}_{t - 1})} - \sqrt{\bar{\alpha}_{t - 1}(1 - \bar{\alpha}_{t})} \ \big\rvert}}$.

Notably, quantifying accuracy in the velocity-space only requires a simple linear transformation. 
Given $\bm{\acc}$ in \Cref{eq:latent_video_accuracy_as_velocity}, we train $\mbf{\epsilon}_{\theta}$ and $\mbf{f}_{\phi}$ without the reverse diffusion process for greater training efficiency.
\end{remark}

\subsection{Proper Scoring Rule}
\label{app:proper_scoring_rule}
For a random variable ${Y}$ following the distribution $P(Y)$, a scoring rule $S$ evaluates a prediction $q$ of the probability distribution of $Y$ by assigning a real-valued score, which could be interpreted as a penalty or reward, providing a measure of the quality of the predicted distribution. A scoring rule is proper if:
\begin{equation}
    \underset{y\sim p(Y)}{\mathbb E}\  S(p(Y), y)\leq \underset{y\sim p(Y)}{\mathbb E}\  S(q,  y),
    \label{eq:proper_scoring_rule}
\end{equation}
for all $q$~\cite{gneiting2007strictly}. Note that a proper scoring rule assesses a larger penalty for all predictions that are not equal to the underlying probability distribution of $Y$.
Intuitively, the proper score is minimized when the predicted distribution $q$ matches the true probability of $Y$. The scoring rule is strictly proper if equality holds in~\Cref{eq:proper_scoring_rule}  if and only if $p(Y)=q$.
Some examples of proper scoring rules include the Brier Score (BS)~\cite{glenn1950verification}, Cross Entropy (CE), and Binary Cross Entropy (BCE), discussed in the paper. 

In addition, we note that the expected calibration error (ECE) and maximum calibration error (MCE) are the standard metrics used to measure deviation from perfect calibration:
\begin{align}
    \label{eq:ece}
    \text{ECE} &\coloneqq\sum_{m=1}^M\frac{|B_m|}{n}\left|\text{acc}(B_m)-\text{conf}(B_m)\right|, \\
    \text{MCE} &\coloneqq\max_{m \in \{1,..., M\}}\left|\text{acc}(B_m) - \text{conf}(B_m)\right|,
\end{align}
where $B_m$ represents bin $m$ with cardinality $|B_m|$, and $n$ represents the total number of samples across all bins.

\section{Proofs}
\label{app:proofs}

\propdecomp*

\begin{proof}
    The proof of this proposition follows immediately from the definition of a proper scoring rule in~\Cref{eq:proper_scoring_rule}.
    Hence, from~\Cref{eq:proper_scoring_rule}, the minimum of the right-hand side (RHS) of~\Cref{eq:proper_scoring_rule} is unique and in particular, is attained when ${\mbf{q} = p(\mbf{y})}$. By optimizing $\phi$ to minimize the RHS of~\Cref{eq:proper_scoring_rule}, we have that ${\hat{\mbf{q}} \rightarrow p(\mbf{y})}$, assuming convergence, which is guaranteed under relatively weak conditions \cite{chen2019convergence, he2023convergence}.
   Further, we have that:
    \begin{equation}
        \label{eq:proof_uncertainty_calibration}
        \mbb{P}[\mbf{Y} = \mbf{1} \mid \mbf{Q} = \hat{\mbf{q}}] = \Expect[\mbf{Y} \mid \mbf{Q} = \hat{\mbf{q}}] = \hat{\mbf{q}},
    \end{equation}
    where the last equality follows from the fact that ${\hat{\mbf{q}} = p(\mbf{y})}$, upon convergence.
    The result in \Cref{eq:proof_uncertainty_calibration} indicates calibration of the predicted confidence $\hat{\mbf{q}}$.
\end{proof}